\documentclass[msom,nonblindrev]{informs4_hide} %

\OneAndAHalfSpacedXI %

\usepackage{amsmath,amsfonts,amssymb}
\usepackage{oubraces}
\usepackage{mathtools}
    \allowdisplaybreaks
\usepackage{bm}
\usepackage[mathscr]{euscript}
\usepackage[inline]{enumitem}
\usepackage{xcolor}
\usepackage{graphicx}
\usepackage{tabularx}
\usepackage{multirow}
\usepackage{bigstrut}
\usepackage{booktabs}
\usepackage{microtype}
\usepackage{fix-cm}
\usepackage{algorithm}
\usepackage{algpseudocode}
\usepackage{diagbox}

\usepackage{physics}
\usepackage{xfrac}
\usepackage{nicematrix}
\usepackage{caption}
\usepackage[hidelinks]{hyperref}
\usepackage{soul}

\usepackage{dsfont}
\usepackage{pifont}

\DeclareMathOperator{\LLM}{\mathrm{LLM}}

\DeclareMathOperator{\TV}{\mathrm{TV}}

\DeclareMathOperator{\Unif}{\mathrm{Unif}}
\DeclareMathOperator{\GP}{\mathrm{GP}}
\DeclareMathOperator{\RF}{\mathrm{RF}}

\DeclareMathOperator{\as}{\mathrm{a.s.}}
\DeclareMathOperator{\Lip}{\mathrm{Lip}}

\DeclareMathOperator{\preA}{\mathsf{preA}}
\DeclareMathOperator{\postA}{\mathsf{postA}}

\DeclareMathOperator{\EMS}{\texttt{EMS}}
\DeclareMathOperator{\FP}{\texttt{FP}}
\DeclareMathOperator{\WS}{\texttt{WS}}
\DeclareMathOperator{\TECH}{\texttt{TECH}}
\DeclareMathOperator{\RT}{\texttt{RT}}
\DeclareMathOperator{\MKT}{\texttt{MKT}}
\DeclareMathOperator{\WTP}{\texttt{WTP}}
\DeclareMathOperator{\QUT}{\texttt{QUT}}
\DeclareMathOperator{\AD}{\texttt{AD}}
\DeclareMathOperator{\MPF}{\texttt{MPF}}
\DeclareMathOperator{\RPF}{\texttt{RPF}}
\DeclareMathOperator{\SCWF}{\texttt{SCWF}}
\DeclareMathOperator{\FISC}{\texttt{FISC}}
\DeclareMathOperator{\ENV}{\texttt{ENV}}
\DeclareMathOperator{\SUBSIDY}{\texttt{SUBSIDY}}

\DeclareMathOperator{\Tone}{(\mathrm{I})}
\DeclareMathOperator{\Ttwo}{(\mathrm{II})}
\DeclareMathOperator{\Tthree}{(\mathrm{III})}
\DeclareMathOperator{\Tfour}{(\mathrm{IV})}

\newcommand{\BB}{\mathbb B}

\newcommand{\RR}{\mathbb R}

\DeclareMathOperator{\pr}{\mathbb P}
\DeclareMathOperator{\E}{\mathbb E}

\newcommand{\msfC}{\mathsf C}

\newcommand{\msfM}{\mathsf M}

\newcommand{\msfP}{\mathsf P}
\newcommand{\msfR}{\mathsf R}

\newcommand{\msfLambda}{\mathsf \Lambda}

\newcommand{\scrF}{\mathcal F}
\newcommand{\scrG}{\mathcal G}

\newcommand{\scrM}{\mathcal M}
\newcommand{\scrN}{\mathcal N}
\newcommand{\scrO}{\mathcal O}

\newcommand{\scrR}{\mathcal R}

\newcommand{\scrT}{\mathcal T}

\newcommand{\scrW}{\mathcal W}
\newcommand{\scrX}{\mathcal X}
\newcommand{\scrY}{\mathcal Y}

\DeclareFontFamily{U}{mathx}{\hyphenchar\font45}
\DeclareFontShape{U}{mathx}{m}{n}{
  <-> mathx10
}{}
\DeclareSymbolFont{mathx}{U}{mathx}{m}{n}
\DeclareMathAccent{\widecheck}{0}{mathx}{"71}
\usepackage{accents}

\newcommand{\wtG}{\widetilde G}

\newcommand{\wtC}{\widetilde C}

\let\emptyset\varnothing

\usepackage{natbib}
 \bibpunct[, ]{(}{)}{,}{a}{}{,}%
 \def\bibfont{\small}%
\TheoremsNumberedThrough     %
\ECRepeatTheorems

\EquationsNumberedThrough    %

\newtheorem{subassumption}{Assumption\normalfont}[assumption]

\newenvironment{assumption*}
    {\ifnum\value{subassumption}=0 \stepcounter{assumption}\fi\subassumption}
    {\endsubassumption}
\newenvironment{assumption+}[1]
    {\renewcommand{\thesubassumption}{#1}\subassumption}
    {\endsubassumption}

\MANUSCRIPTNO{}

\begin{document}

\RUNTITLE{Optimizing Service Operations via LLM-MAS}

\TITLE{Optimizing Service Operations via LLM-Powered Multi-Agent Simulation}

\ARTICLEAUTHORS{
\AUTHOR{Yanyuan Wang, Xiaowei Zhang}
\AFF{Department of Industrial Engineering and Decision Analytics, The Hong Kong University of Science and Technology, Clear Water Bay, Hong Kong SAR, \EMAIL{yanyuan.wang@connect.ust.hk}, \EMAIL{xiaoweiz@ust.hk}}
}

\ABSTRACT{Service system performance depends on how participants respond to design choices, but modeling these responses is hard due to the complexity of human behavior. We introduce an LLM‑powered multi‑agent simulation (LLM‑MAS) framework for optimizing service operations. We pose the problem as stochastic optimization with decision‑dependent uncertainty: design choices are embedded in prompts and shape the distribution of outcomes from interacting LLM-powered agents. By embedding key numerical information in prompts and extracting it from LLM‑generated text, we model this uncertainty as a controlled Markov chain. We develop an on‑trajectory learning algorithm that, on a single simulation run, simultaneously constructs zeroth-order gradient estimates and updates design parameters to optimize steady-state performance. We also incorporate variance reduction techniques. In a sustainable supply chain application, our method outperforms benchmarks, including blackbox optimization and using LLMs as numerical solvers or as role‑playing system designers. A case study on optimal contest design with real behavioral data shows that LLM‑MAS is both as a cost‑effective evaluator of known designs and an exploratory tool that can uncover strong designs overlooked by traditional approaches.
}


\maketitle

\section{Introduction} \label{sec:intro}

Service systems are driven by decisions and interactions of people. Examples include multi-echelon supply chains \citep{Cohen16GSC}, innovation contests \citep{Chen20}, and digital marketplaces \citep{BandiCohenRay24}. A central task in these settings is to design policies, incentives, and operational levers that deliver strong performance. These design parameters, including sustainability rules, incentive structures, and platform interfaces, shape how individuals decide and communicate. Individual choices, such as adopting a regulation, entering a contest, or selecting a product, then aggregate into outcomes such as cost effectiveness, efficiency, and externalities. Yet human behavior often deviates from standard analytical and empirical predictions, because decisions reflect how people form beliefs, assess risks, and make trade-offs in contexts that are hard to observe. Effective designs should therefore account for behavioral factors to capture how people  respond to design choices, as these responses ultimately determine system performance \citep{Huang13}.

Despite broad recognition of their importance in operations management (OM), human decision processes remain hard to model \citep{Donohue20}. Analytical models, such as utility-based formulations, use simplified assumptions and a small set of parameters to maintain tractability. This helps abstraction but often misses key features of behavior. Calibrating these models for practical use also requires data that are costly to collect in OM settings, whether through lab studies or field experiments. A further challenge is that behavior is not fully rational. Decisions are shaped by cognitive limits and by time and attention constraints during the decision process \citep{Donohue18}. These factors reduce the accuracy of traditional modeling and analysis approaches and limit their usefulness for system design in practice.

We take a different approach. Instead of building analytical models and calibrating them with costly data, we use large language models (LLMs) to simulate human decisions. Pretrained on extensive human-generated text, LLMs capture statistical patterns in how people reason, communicate, and make decisions. Recent studies report human-like behavior from LLMs across economics \citep{ChenLiuShanZhong23}, marketing \citep{GoliSingh24}, psychology \citep{Demszky23using}, and political science \citep{ArgyleBusbyFuldaGublerRyttingWingate23}. 
Most relevant to our paper,
\cite{Chen25} show that 
GPT models largely reproduce the decision biases as documented by \cite{Davis18} in canonical OM settings  (e.g., inventory management, sourcing, forecasting, and queueing). 

Beyond generating human‑like behavior (the response), LLMs also interpret unstructured, narrative-rich information that traditional numerical simulators struggle to handle (the comprehension).
Many operational outcomes hinge on qualitative or strategic text such as advertisements, seller messages, customer reviews, recommender explanations, and chatbot replies. As native text processors, LLMs can both analyze and generate language that responds to nuances in phrasing, sentiment, and dialogue context.
This opens new possibilities for
experimentation on system designs in digital operations such as pricing informed by reviews, customer-service automation, and AI-moderated marketplaces.

We treat LLMs as agents within a service system and let them interact under the given system design. Their interactions generate system-level outcomes, which determine the design's performance. 
We refer to this approach as LLM-powered multi-agent simulation (LLM-MAS). It follows the agent-based modeling paradigm, where micro-level behaviors aggregate into macro-level dynamics, making it well suited for systems with complex interdependencies \citep{AxtellFarmer25}.\footnote{Agent-based models usually use handcrafted rules or simple functions to govern behavior and interaction. Such rule-based agents oversimplify behavior and generalize poorly across settings. Recent advances apply reinforcement learning (RL) to model agent behavior, enabling agents to learn more complex decision rules. However, RL trains agents to maximize cumulative rewards, which conflicts with the evidence that people are not perfect optimizers \citep{AllonCohenSinchaisri23}. It is also challenging to specify a reward function that faithfully reflects human decision-making. These limitations motivate using LLMs as more expressive agent models.}

We illustrate the LLM-MAS approach to system design with the following example. As AI reshapes supply chain management \citep{CohenDai26}, a policymaker seeks to develop sustainable policies that balance economic viability, environmental impact, and supply chain welfare \citep{Krass13Tax,Cohen16GSC,Bian20Subsidy}. 
This is our running example. 
We also present in Section~\ref{sec:case} a case study on optimal design of innovation contests \citep{Chen20} using real behavioral data.

\begin{example}[Sustainable Supply Chain Management] \label{ex:scm}
Consider a three-echelon supply chain with manufacturer ($\msfM$), retailer ($\msfR$), and consumer ($\msfC$); see Figure~\ref{fig:LLM-MAS-SCM}. The policymaker designs sustainable policies with two parameters: a carbon tax on the manufacturer ($\theta_1$) and a consumer subsidy ($\theta_2$).
For any given policy $\theta=(\theta_1,\theta_2)$, we simulate the decisions of all three players using LLM-MAS rather than specifying their behavior analytically.  
The manufacturer chooses a wholesale price ($\WS$) and an investment level in low-carbon technology ($\TECH$);
the latter determines the system's carbon emissions ($\EMS$). 
The manufacturer also discloses the carbon footprint ($\FP$) to the retailer, which is defined as the percentage reduction in carbon emissions relative to a baseline. 
The retailer sets the retail price ($\RT$) and the marketing budget ($\MKT$), then combines the price with the carbon footprint to create an advertisement ($\AD$) funded by the budget. 
After viewing the ad, 
the consumer chooses purchase quantities ($\QUT$) and reveals willingness to pay ($\WTP$). 
The ad narratives are generated by an auxiliary LLM, and the  consumer interprets the ad and acts, highlighting the strength of LLM-MAS at processing unstructured information end to end. 
The system state  ($\xi$) includes all agents' decisions (i.e., $\TECH, \WS, \MKT, \RT, \QUT, \WTP$) and the resulting outcome (i.e., $\EMS$). The policymaker evaluates each policy using 
$F(\theta; \xi) = -\big(\SCWF(\theta;\xi) - \FISC(\theta;\xi) - \ENV(\xi)\big)$, 
which balances supply chain welfare ($\SCWF$), fiscal costs ($\FISC$), and environmental externalities ($\ENV$). The goal is to choose policy parameters that optimize long-run system performance over multiple rounds of interactions among the three agents.
\end{example}

\begin{figure}[ht]
    \FIGURE{
    $
    \begin{array}{c}
    \includegraphics[width=0.7\textwidth]{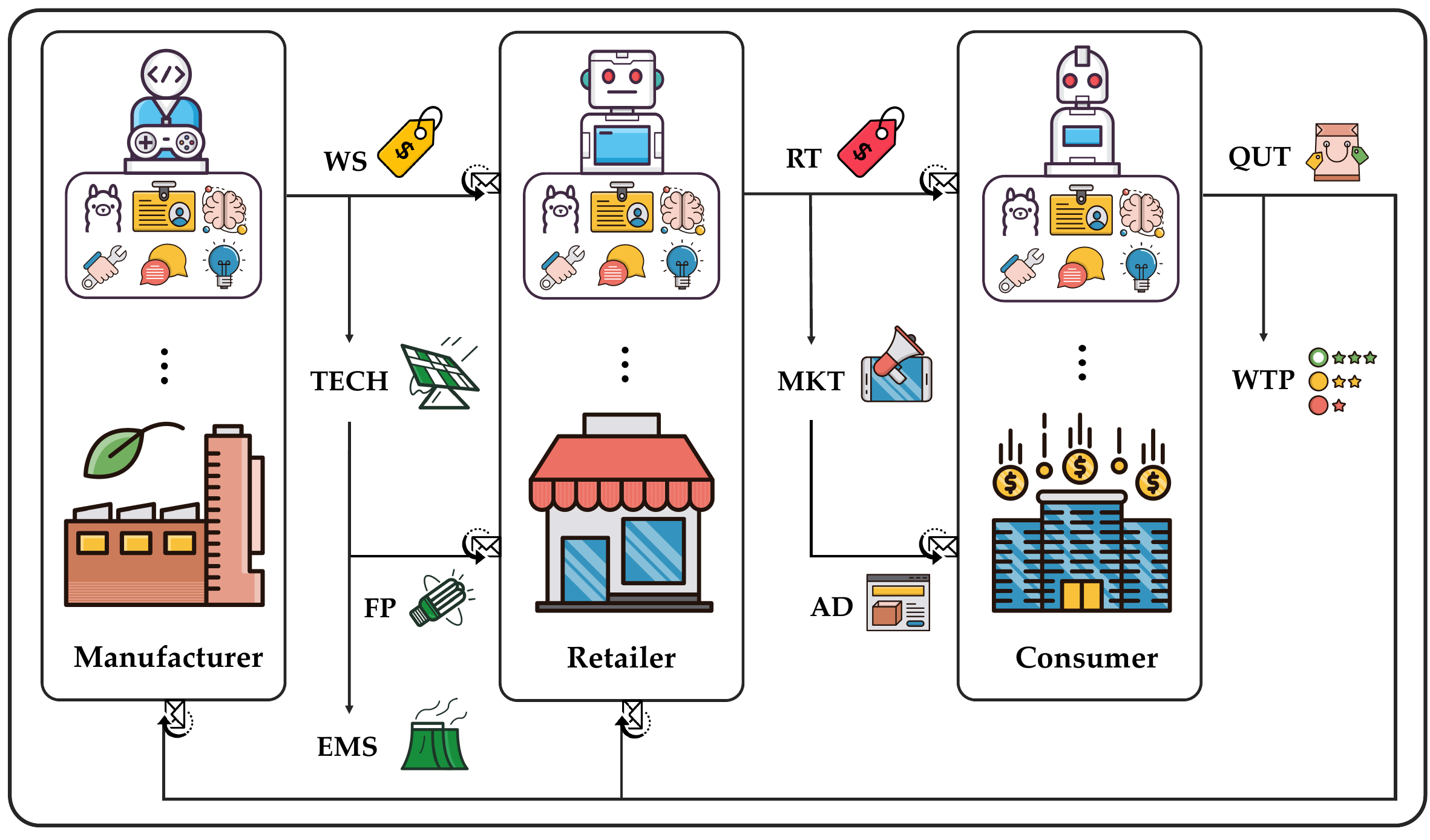} 
    \end{array}
    $
    }
    {LLM-MAS for Sustainable Supply Chain Management \label{fig:LLM-MAS-SCM}
    }
    {%
    } 
\end{figure}

Inspired by this example, we pose the following question:
\begin{center}
    \emph{``How can LLM-MAS be integrated into system design optimization?''}
\end{center}
To address this question, we propose a principled framework (Figure~\ref{fig:framework}). 
A system designer seeks to optimize design choices using the dynamics generated via LLM-MAS, where LLM-powered agents serve as \emph{silicon proxies}. 
This framework casts the problem as stochastic optimization with uncertainty driven by LLM-generated responses, facilitates the development of optimization algorithms, and accommodates any off-the-shelf LLM. 
We use these models as is, assuming that LLMs can accurately simulate human behavior in the application of interest, so we do not fine-tune or post-train models within LLM-MAS. (Calibrating models to improve behavioral fidelity is outside the scope of this paper.)
This assumption is supported by our case study on optimal contest design, where the LLM-MAS, implemented with an open-source model, replicates the lab experiments of \cite{HuangZhangZhang26} involving hundreds of human participants.

\begin{figure}[t]
    \FIGURE{
    $
    \begin{array}{c}
    \includegraphics[width=0.8\textwidth]{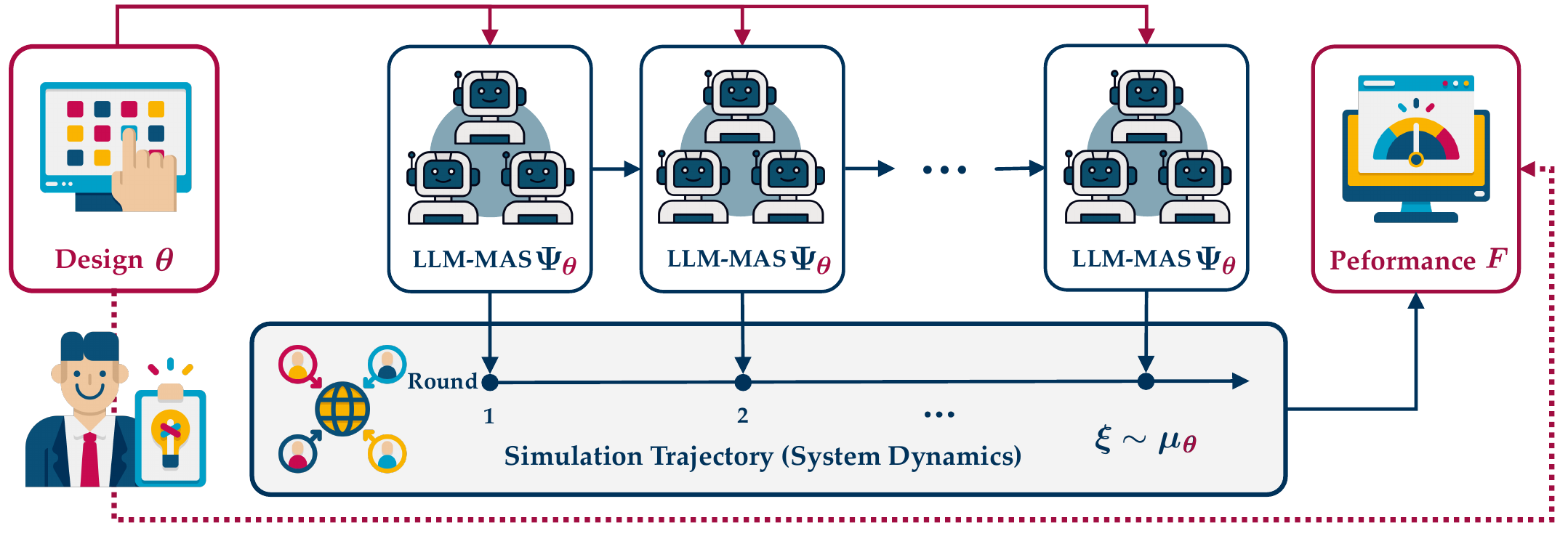} 
    \end{array}
    $
    }
    {LLM-MAS for System Design Optimization \label{fig:framework}
    }
    {\begin{enumerate*}[label=(\roman*)]
        \item $\Psi_{\theta}$ is the one-round simulation of system state $\xi$.
        \item $\mu_\theta$ is the stationary  distribution for a given design $\theta$.
    \end{enumerate*}} 
\end{figure}

This framework serves as a digital twin that  links theory to practice  and complements existing approaches to system design optimization along two dimensions.
\begin{enumerate*}[label=(\roman*)]
    \item \emph{Evaluation (bridging theory and reality).} 
    It narrows the gap between elegant theoretical models and complex real systems. By testing theoretical solutions in LLM-MAS, it relaxes idealized assumptions and checks robustness under realistic conditions. This is particularly valuable in domains where human behavior shapes outcomes, clean models fall short, and behavioral data {are} costly to collect.
    \item \emph{Exploration (providing a silicon sandbox).} 
    It enables structured exploration of alternative  designs in a controlled, synthetic  environment. 
    The testbed helps find designs that perform well when human behavior plays a crucial role, which is often analytically intractable, while reducing the time and resources needed for lab experiments.
\end{enumerate*}

\subsection{Main Contributions}

\subsubsection{Modeling.}
We formulate system design via LLM-MAS as a stochastic optimization problem with decision-dependent uncertainty, represented as a controlled Markov chain. 
Design parameters enter as prompt variables for each LLM-powered agent, which shift the distribution of token outputs.
This token-level randomness yields text responses that encode each agent's behavior, and over multiple rounds of interactions, these behaviors aggregate to determine system performance. 
We construct a system-level Markovian state by identifying compact, meaningful variables from sequences of LLM prompts and responses that constitute the agents' observations, communications, and actions.
This representation clarifies how design choices drive the evolution of uncertainty in LLM-MAS and provides a foundation for efficient learning algorithms for system design optimization.

\subsubsection{Algorithms.}
We develop a new algorithm to address two challenges from the resulting stochastic optimization problem.
The first is gradient computation.
With decision‑dependent uncertainty, the gradient depends on the conditional distribution of the system state given the design.
In LLM‑MAS, this distribution is intractable. 
The uncertainty varies with design parameters in complex ways and arises from numerous sources, including hardware configurations.
We use a zeroth‑order method that constructs gradient estimates from function evaluations, which removes the need to explicitly model the conditional distribution.

The second challenge concerns steady‑state evaluation. 
A direct approach needs many trajectories of the controlled Markov chain, each with many rounds of LLM-powered agents' interaction, to approximate the steady-state distribution. 
This is computationally prohibitive considering LLM inference latency. 
We propose an on‑trajectory learning scheme which, within a single simulation run, estimates gradients and updates design parameters simultaneously, but on two timescales for convergence. 
It also learns from the entire state path, not only the terminal state as in the standard multi‑trajectory method, thereby reducing LLM queries.

On the theory side, we establish the convergence of the on-trajectory learning algorithm through the lens of stochastic approximation (SA) with adaptive Markovian data.
This setting arises because the transition distribution is not fixed but depends on the current iterate, which is updated from data collected thus far. 
We use a Poisson equation approach to capture the temporal dependence in the data.
We then exploit the smoothness of its solution to control the gap to the steady state and to bound the error of our zeroth‑order estimator for the stochastic gradient under the transition distribution.

Furthermore,  we incorporate two variance reduction techniques (guided perturbation and residual feedback) to enhance the algorithm's efficiency.
One exploits known structure in the performance function. The other applies when such structure is unavailable and instead relies on information from the algorithm's iterates.

\subsubsection{Empirics.}
We demonstrate the practical value of our LLM-MAS framework and algorithm through two applications: sustainable supply chain management and innovation contest design.
In the supply chain setting, we compare our methods against several benchmarks: Bayesian optimization, which treats the problem as a black box, an LLM-as-a-Solver approach that prompts the LLM to perform numerical optimization directly, and an LLM-as-a-Designer approach in which the LLM functions as the decision-maker, drawing on all agents' actions simulated via LLM-MAS.
Across these benchmarks, our algorithm and its variance-reduction variant consistently achieve superior performance. The results offer two key insights. First,  modeling uncertainty as a controlled Markov chain allows us to exploit the probabilistic structure of LLM-MAS for greater computational efficiency. Second, relying on LLMs alone is insufficient for system design, which needs a rigorous stochastic formulation and problem-specific numerical optimization algorithms.
This findings align with a growing view in supply chain management that disruptive GenAI tools are complements, not substitutes, for traditional OM methods \citep{
CohenDai26}.

In the contest design setting, we use behavioral data from a recent lab experiment \citep{HuangZhangZhang26} to show the value of LLM-MAS for both evaluating given designs and searching for alternatives. This is especially important when human behavior matters. We show that LLM-MAS largely replicates heterogeneous behaviors in the lab: whether to enter and, conditional on entry,  how much effort to exert as a function of the participant's ability. This similarity holds across contest designs with different entry fees, reserves, and shared prizes. Both the lab data and LLM-MAS deviate markedly from game-theoretic predictions, which assume risk-neutral, fully rational agents and yield closed-form optimal designs. 
Therefore, LLM-MAS can reduce the cost and time of design evaluation relative to conducting new lab experiments. 

A key by-product of LLM-MAS is the rich reasoning-process data that reveal agents' decision logic, offering a level of interpretability rarely attainable with other AI tools such as RL or with lab and field experiments that lack post‑hoc interviews. For instance, reasoning traces from our contest‑design case study can help designers determine whether a design continues to elicit effort as the prize budget scales.  
This, in turn, enables more reliable and cost‑efficient evaluation of system configurations.

Furthermore, 
we optimize contest parameters to maximize revenue. The LLM-MAS solution differs substantially from the game-theoretic recommendation, which ties the reserve and shared prize to the entry fee and favors higher entry fees. 
In contrast, our solution decouples the three design parameters, recommends a moderate entry fee, and sets reserve and shared prize well below the game-theoretic levels. It also achieves much higher revenue than the game-theoretic optimum. 
These results show that LLM-MAS can discover viable designs that traditional models would discard and can guide subsequent lab or field tests.

\subsection{Related Works}

\subsubsection{LLM-MAS.}

Our work connects to a fast-growing literature that uses LLMs to simulate and study human behavior in business, economics, and social sciences. The appeal is clear: LLMs can lower the cost of collecting behavioral data by orders of magnitude and enable experiments where human studies are infeasible or raise ethical concerns. Even if LLM-generated data may not perfectly replicate human data, they can still help researchers explore and vet new ideas. A key thread in this literature builds high-fidelity simulation environments to answer what-if questions, predict how designs affect outcomes, and support counterfactual analysis, hypothesis testing, and theory prototyping \citep{Park22Simulacra}. 
For instance, \cite{Shirani25} employ LLM-powered agents to reproduce Facebook's randomized controlled trials and replicate the finding that displaying peers' behaviors on social media platforms have a political mobilization effect that increases voter turnout. \cite{Allouah25ACES} build an LLM-powered sandbox of e-commerce platforms to study how platform levers and product attributes influence AI-agent buyers' actions, providing implications for both platform operators and sellers.
For overviews, see \cite{GaoLanLiYuanDingZhouXuLi24}, 
\cite{AnthisLiuRichardsonKozlowskiKochBrynjolfssonEvansBernstein25}, and \cite{FengDouLiWangWangGuoMaKong25}. 
Our contribution advances this line of research from \emph{design evaluation} to \emph{design optimization}: we use LLM-MAS across design choices not only to assess outcomes but also to optimize performance.

\subsubsection{Intersection of LLMs and OR.}

This research aligns with recent efforts at the intersection of LLMs and operations  \citep{DaiSwaminathan26}. Prior work broadly follows two directions: LLMs for operations research (OR) and OR for LLMs.
Most studies in the first direction use LLMs for mathematical modeling by translating natural-language problem statements into formal formulations, with applications to deterministic optimization \citep{Huang25ORLM},
robust optimization \citep{BertsimasMargaritis24}, and dynamic programming \citep{Zhou25DPLM}. 
Our study belongs to this stream but differs in purpose. Instead of using LLMs to convert text into stochastic programs for standard numerical solvers, we formulate the mathematical model ourselves, use LLMs to simulate uncertain dynamics, and solve the problem with a tailored algorithm that reflects LLM-specific features.

The second direction, OR for LLMs, applies OR methods to improve LLM performance, treating LLMs as an application domain instead of as tools for OR. 
This line of research mainly focuses on inference-time efficiency, 
typically through scheduling and batching to manage the dual queues of token prefill and decoding \citep{Jaillet25InfOpt,Li25InfOpt,Ao25InfOpt}; see also \cite{Mitzenmacher25} for a recent survey. 
Our framework also accounts for LLM characteristics, but our aim is to leverage LLMs for simulation-based optimization, not to enhance the models' internal capabilities.

\subsubsection{Zeroth-Order Optimization and SA.}
The proposed algorithm falls under zeroth-order methods for stochastic optimization when gradients are unavailable or costly to compute.
These methods update solutions much like gradient-based approaches but estimate gradients using function evaluations; see \cite{PrashanthShalabh25} for an introduction.
Most prior studies focus on settings where the distribution of the random variable does not depend on the decision. 
In contrast, our optimization problem  has decision-dependent uncertainty \citep{DrusvyatskiyLin23DDU}, captured by the stationary distribution of a controlled Markov chain, which introduces complex intertemporal dependencies into the algorithm. 

Analyses of zeroth‑order methods 
typically draw on SA \citep{KushnerYin03,GhadimiLan13,BalasubramanianGhadimi22Complexity}.
We follow this practice, but 
our on-trajectory learning algorithm gives rise to SA with adaptive Markovian data. 
This type of SA has gained much attention due to its application to RL algorithms with policy improvement \citep{KondaTsitsiklis03,ZhangHu22,CheDongTong26,LiLiangChenZhang26}.  
However, these studies usually assume access to first‑order information, in contrast to our setting where gradients are intractable and must be estimated using zeroth-order methods.

\smallskip

The remainder of this paper is organized as follows. Section \ref{sec:prob} formulates system design via LLM-MAS as a stochastic optimization problem with decision-dependent uncertainty driven by a controlled Markov chain. Section \ref{sec:llm-mas} explains how we extract the Markov chain's state variables from text exchanges between LLM-powered agents. Section \ref{sec:algo} presents the on-trajectory learning algorithm to address challenges specific to LLM-MAS and shows its asymptotic convergence.
Section~\ref{sec:enhance} incorporates two variance reduction techniques to enhance efficiency. 
Section \ref{sec:scm} presents numerical experiments in a supply chain setting to compare algorithm performance. 
Section \ref{sec:case} provides a case study on optimal contest design using real human behavioral data. 
Section \ref{sec:concl} concludes and the e-companion contains technical proofs.

\section{Problem Formulation} \label{sec:prob}

Consider a service system with design parameter vector $\theta\in \Theta \subseteq \RR^{d}$. 
We evaluate its performance through LLM-MAS.
Let $\xi \in \Xi$ be the random variable that captures  the uncertainty in the system. 
The system performance is measured by a cost function $F(\theta;\xi)$. Our goal is to find an optimal design  that minimizes the expected cost with respect to the distribution of $\xi$.  

In the supply chain example (Example~\ref{ex:scm}), the design $\theta$ consists of a carbon tax on the manufacturer and a consumer subsidy. 
The uncertainty $\xi$ faced by the system designer includes agent behavior and system outcomes:
the manufacturer's wholesale price and investment in low-carbon technology, the retailer's retail price and marketing budget, consumer purchase quantities and willingness to pay, and the system's carbon emissions.
The cost function $F$ reflects supply chain welfare, fiscal spending, and environmental externalities.

What sets this problem apart from most stochastic programs is how the distribution of $\xi$ is generated and how it depends on $\theta$. 
In LLM-MAS, $\theta$ is embedded in the prompts, shapes token probabilities, affects agent actions, and alters system outcomes. As a result, the distribution of $\xi$ varies with $\theta$. 
We therefore formulate the system design problem as stochastic optimization with \emph{decision-dependent} uncertainty: 
\begin{align}
    \min_{\theta \in \Theta} \bigg\{ f(\theta) \coloneqq \E_{\xi \sim \mu_\theta}[F(\theta;\xi)] \bigg\},  \label{eq:so}
\end{align}
where $\mu_\theta$ denotes the distribution of $\xi$ for a given design $\theta$. 
Here, the term ``decision'' refers to the system designer's choice $\theta$, rather than the agents' actions simulated by LLMs.
In what follows, we detail the mechanism underlying the uncertainty within LLM-MAS.

\subsection{Token-Level Uncertainty} \label{sec:token-uncertainty} 

LLMs produce random outputs because they are autoregressive probabilistic models.
At each step, the model predicts a distribution over the next token given the context so far.
It computes logits over the vocabulary (i.e., the fixed set of tokens the model can output, typically subword pieces defined by a tokenizer), applies a softmax mapping to obtain probabilities, and then selects the next token with a decoding rule.

In practice, decoding is rarely deterministic. 
Popular LLM families such as GPT, LLaMA, and DeepSeek use sampling-based decoding with configurations such as temperature, top-$k$, or top-$p$; see \cite{build-llms-from-scratch-book} for details.
These hyperparameters regulate the randomness in token generation. For instance, temperature rescales the logits before softmax, with higher values spreading the distribution to increase variation. 
Even if two runs of simulation start with the same prompt, a small sampling difference in an early token can alter the context for the subsequent autoregressive generation, and thus propagates through later token distributions. As a result,  the two runs drift apart as the token sequence extends. 

\begin{remark} \label{remark:uncertainty}
Setting temperature to zero does not ensure identical outputs because several system factors introduce nondeterminism. 
First, many LLMs use low‑precision quantization (8‑bit or 4‑bit) for weights and sometimes activations, which coarsens values, increases near‑ties, and makes tie‑breaking more frequent.
Tie-breaking is not always consistent across runs, and an early flip can cascade into a different completion \citep{DettmersPagnoniHoltzmanZettlemoyer23}.
Second, floating-point arithmetic is non-associative on parallel GPUs. Operations such as summation may run in different orders, 
which changes rounding and can flip the argmax when logits are close 
\citep{Dao22FlashAttention}. 
Third, many users rely on third-party LLM services to access proprietary models such as GPT and Claude rather than hosting them locally. 
Service providers often batch prompts from multiple users to process together and boost throughput. 
The mix of prompts in a batch can change over time, and those small changes can nudge the model toward a different output \citep{Rajbhandari22deepspeed}. 
\end{remark}

\subsection{Decision Dependence}\label{sec:decision-depend}

In LLM-MAS, each agent's behavior is simulated by prompting an LLM with the information a comparable human participant would have. The prompt includes the current system design $\theta$, the agent's persona, relevant behavior of other agents, and any private context. The LLM processes this prompt and generates tokens that describe the agent's behavior. Agents then interact in sequences specified by the simulation environment, which determine event order and information flow. Finally, the simulated behaviors are aggregated into a system-level outcome $\xi$, which contains the quantities used in the performance function $F$.

Because $\theta$ is embedded in each prompt, changing $\theta$ alters token probabilities. Those shifts change the distribution of agent actions extracted from the LLM outputs. Through agent interactions, these changes accumulate into the system state $\xi$. This hierarchical propagation explains how incorporating $\theta$ in prompts induces a system-level distribution $\mu_\theta$ over $\xi$.

To illustrate, consider the supply chain setting in Example~\ref{ex:scm}. 
Consumer subsidy ($\theta_2$), one of the design parameters, is included in the prompt to the consumer agent: ``\texttt{You receive a subsidy of SUBSIDY dollars for each unit of product you purchase.}''
Here, \texttt{SUBSIDY} is replaced with the actual value of $\theta_2$ when the LLM processes the prompt to generate the consumer's purchasing behavior.  
Increasing $\theta_2$ makes the agent more likely to produce tokens that reflect higher purchase volume, 
which, in turn, changes the distribution of the system-level outcome $\xi$. 
See Figure~\ref{fig:dist} for details.

\begin{figure}[ht]
    \FIGURE{
    $
    \begin{array}{c}
    \includegraphics[width=0.7\textwidth]{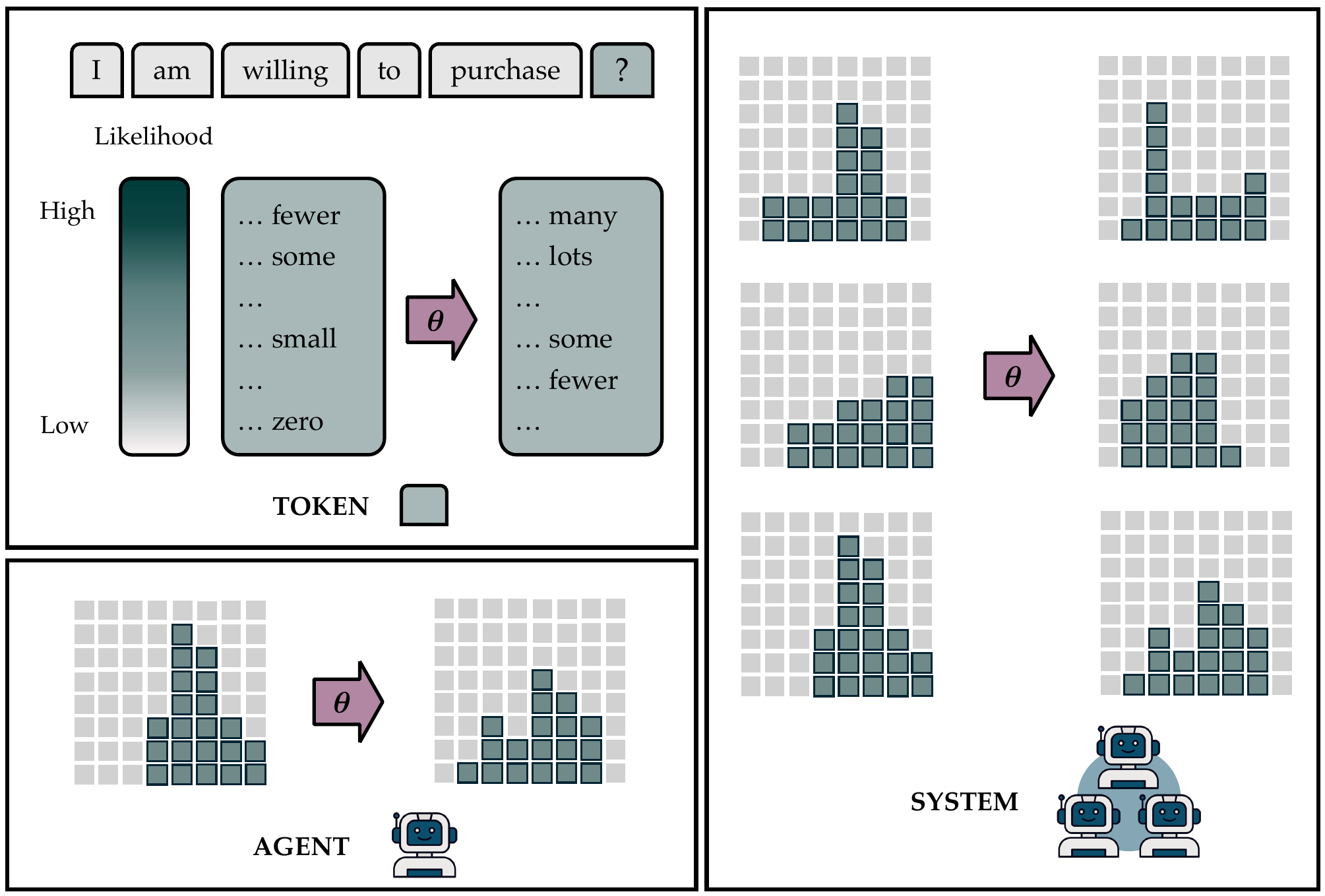} 
    \end{array}
    $
    }
    {How $\theta$ Affects $\xi$: Hierarchical Propagation from Tokens to Agents to System \label{fig:dist}
    }
    {The distributions are for illustrative purposes only. A higher consumer subsidy ($\theta_2$) embedded in the prompt makes the consumer agent more likely to generate a token that indicates higher purchase volume following the phrase ``I am willing to purchase''.  This essentially modulates the distribution over possible next tokens, shifting the probability mass from lower-intensity terms (e.g., ``fewer'' and ``some'') to higher-intensity ones (e.g., ``many'' and ``lots'').} 
\end{figure}

\subsection{Steady-State Performance}

When using LLM-MAS to simulate and optimize service systems, 
the system state $\xi$ is typically not produced in a single step. 
Instead, it emerges over multiple rounds of simulated interactions among agents.
The state is therefore best viewed as a stochastic process, taking the value $\xi^{(t)}$ in round $t$. 
This mirrors reality: human participants need time to interpret new rules, observe others' responses, and adjust their own strategies. 
Only after this period of learning and adaptation do behaviors and interactions converge to stable patterns, at which point the effects of a system design reveal its long-term performance.

Accordingly, we focus on steady-state performance in this paper. 
Emphasizing the long run captures these adaptation dynamics and reveals the distribution of outcomes after the system has equilibrated. 
Moreover, once a service design is finalized, subsequent modifications are difficult to implement without disruption. 
Optimizing steady-state behavior 
\emph{ex ante} reduces the likelihood of costly revisions and alleviates operational instability. 
In the next section, we show that, by appropriately extracting information from LLM outputs that encode agents' behavior, $\xi^{(t)}$ can be modeled as a controlled Markov chain. 
In line with this focus, the distribution $\mu_\theta$ in \eqref{eq:so} represents the stationary distribution of this Markov chain for a given $\theta$. 
(The methodology developed in this paper also applies to finite-time performance, as in problems with fixed operational horizons or round-based applications such as contests and auctions.)

\section{LLM-MAS as Controlled Markov Chain} \label{sec:llm-mas}
We begin with a general formulation of LLM-MAS and  specify key elements from the agents to define the system state $\xi^{(t)}$: agents' observations of the environment, inter-agent messages, and agents' actions. 
We then revisit the supply chain example to show how to identify these elements in practice. Next, we detail how to prompt LLMs, with particular emphasis on incorporating numerical quantities in the inputs and specifying the required outputs. This setup allows the system state to be extracted from textual LLM outputs round by round, completing the update from $\xi^{(t)}$ to $\xi^{(t+1)}$.

\subsection{System State} \label{sec:simsys}

An LLM-MAS system is represented by the tuple $\langle \scrG, \Psi_\theta, T \rangle$: a directed graph $\scrG$, a one-round simulation operator $\Psi_\theta$, and a time horizon $T$.    
The graph $\scrG$ is defined by an index set $\scrN$ of agents and a  matrix $\scrM$ that encodes directed interaction links, typically representing information, product, or cash flows.

Each agent $i \in \scrN$ has five modules: profile, perception, communication, LLM, and memory. In each simulation round $t=1,\ldots,T$, these modules operate as follows and the order of events is shown in Figure \ref{fig:events-order}.
\begin{enumerate}[label=(\roman*)]
    \item Profile. Defines the agent's functional role $\scrR_i$ and, optionally, heterogeneous attributes $\chi_i$ that capture individual characteristics (e.g., risk aversion) and social preferences (e.g., propensity to collaborate).
    \item Perception. Specifies the agent's observation $O_{i}^{(t)}$. It reflects interactions with the environment, not with other agents. The observation may include individual information (available only to the agent) and/or shared information (available to all agents). In the supply chain example, the manufacturer's carbon emissions are part of the shared observation. In an auction, a bidder's valuation of the item being offered is part of the individual observation. 
    \item Communication. Captures the agent's interactions with others. Let $M_i^{(t)} = \bigcup_{j \in \scrM[i]} (M_{i \to j}^{(t)}, M_{i \gets j}^{(t)})$ denote numerical quantities extracted from information exchange over its links, including outbound messages $M_{i \to j}^{(t)}$ (sent to $j$) and inbound messages $M_{i \gets j}^{(t)}$ (received from $j$), where $j \in \scrM[i]$ and $\scrM[i]$ denotes the set of agents that agent $i$ can communicate with. 
    \item LLM. Serves as the agent's behavior generator, producing the action $A_i^{(t)}$ based on both the system design $\theta$ and the contextual information. In each round, the agent may obtain new information from the environment and from other agents, both before and after acting. 
    Only information available before the action is provided to the LLM as part of the context for reasoning. 
    This context also includes information from the previous round. 
    Information that arrives after the action is recorded and carried to the next round via the memory module. 
    Accordingly, we split the observation $O_i^{(t)}$ into a pre-action part  $O_i^{(t,\preA)}$ and a post-action part $O_i^{(t,\postA)}$. We apply the same practice for communications:  $M_i^{(t)} = (M_i^{(t,\preA)},M_i^{(t,\postA)})$. 
    \item Memory. Stores this round's new information $H_i^{(t)} = (O_{i}^{(t)}, M_i^{(t)}, A_{i}^{(t)})$, including the agent's observation, communication, and action.
\end{enumerate}

\begin{figure}[ht]
    \FIGURE{
    \includegraphics[width=0.9\textwidth]{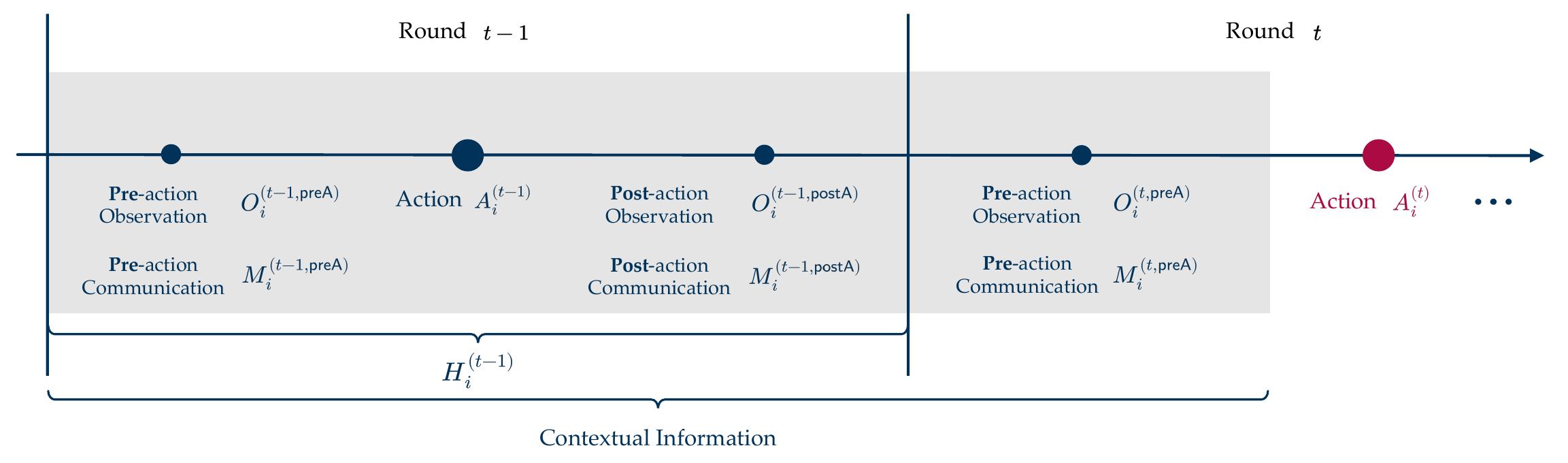} 
    }
    {Order of Events\label{fig:events-order}
    }
    {Contextual information is the inputs, beyond the system design $\theta$, that are provided to the LLM to simulate the agent's action.}
\end{figure}

The system state at the end of round $t$ is the collection of new information from all agents: $\xi^{(t)} = \bigcup_{i \in \scrN} H^{(t)}_i$. 
We emphasize that components of $O_{i}^{(t)}$, $ M_i^{(t)}$, and $A_{i}^{(t)}$ are typically numerical and may be mixed with text.
They must be embedded in LLM prompts and properly extracted from textual outputs (see Section~\ref{sec:prompt}). 
The state transition to the next round is given by the one-round simulation operator $\Psi_\theta$, which maps $\xi^{(t)}$ to $\xi^{(t+1)}$ under the system design $\theta$. 

Iterating $\Psi_\theta$ over the horizon $T$ yields the trajectory $\{\xi^{(t)}:t=0,1,\ldots,T\}$, where each state depends only on its immediate predecessor.
Thus, the LLM-MAS dynamics are Markovian.
With $\theta$ fixed along the trajectory, the system state forms a controlled Markov chain, since the distribution of the agents' behavior simulated by LLMs is affected by $\theta$ (see Section~\ref{sec:decision-depend}).

\subsection{Illustration with the Supply Chain Example} \label{sec:simsys-ex}

In Example~\ref{ex:scm}, the system has three agents with different roles: manufacturer ($\msfM$), retailer ($\msfR$), and consumer ($\msfC$). 
Consider the retailer. In each round $t$, the retailer's action $A_{\msfR}^{(t)}$ consists of a retail price $\RT^{(t)}$ and a marketing budget $\MKT^{(t)}$. These are chosen using the previous memory $H_{\msfR}^{(t-1)}$ and 
the pre-action message from the manufacturer $M^{(t, \preA)}_{\msfR \gets \msfM}$. 
After acting, the retailer calls a tool (simulated by an auxiliary LLM\footnote{This tool-use highlights a key distinction between LLMs as language models and LLM-powered agents. The latter are larger systems that use an LLM component to make plans, call tools (for example, via the Model Context Protocol to connect to external data and software), and execute actions subject to certain constraints  \citep{Cohen25Vision}.}) that uses $\RT^{(t)}$ and $\MKT^{(t)}$ to generate the ad narratives $\AD^{(t)}$.
It then sends $(\RT^{(t)}, \AD^{(t)})$ to the consumer. After the consumer acts, the retailer receives the purchasing quantity $\QUT^{(t)}$.

The retailer has no pre‑action observation, so 
$O_{\msfR}^{(t,\preA)} = \emptyset$.
Its pre‑action communication is the message from the manufacturer $M^{(t, \preA)}_{\msfR \gets \msfM} = (\WS^{(t)}, \FP^{(t)})$. 
Its post-action communications include the message sent to the consumer $M_{\msfR \to \msfC}^{(t,\postA)}  = (\RT^{(t)}, \AD^{(t)})$ and the message later received in return $M_{\msfR \gets \msfC}^{(t,\postA)} = \QUT^{(t)}$. There is no post-action observation. Thus, at the end of this round, the retailer's memory module stores \begin{align*}
    H_{\msfR}^{(t)} = (\overunderbraces{&&&&\br{2}{M^{(t, \postA)}_{\msfR \to \msfC}}}%
{&\WS^{(t)}, & \FP^{(t)}, &\MKT^{(t)}, & \RT^{(t)}, & \AD^{(t)}, & \QUT^{(t)}}%
{& \br{2}{M^{(t, \preA)}_{\msfR \gets \msfM}} &\br{2}{A_{\msfR}^{(t)}} & & \br{1}{M^{(t,\postA)}_{\msfR \gets \msfC}}}).
\end{align*}

The other two agents' observations, communications, and actions are defined similarly, with details in Section~\ref{ec:ex} of the e-companion. 
Table~\ref{tab:context-info-scm} summarizes these results.
\begin{table}[ht]
\TABLE
{Components of System State in the Supply Chain Example \label{tab:context-info-scm}}
{
\begin{tabular}{@{\extracolsep{30pt}} c@{\extracolsep{20pt}} c@{\extracolsep{20pt}} c@{\extracolsep{20pt}}c @{\extracolsep{20pt}} c @{\extracolsep{20pt}} c @{\extracolsep{20pt}} c @{\extracolsep{20pt}} c @{\extracolsep{20pt}}}
    \toprule
    \multirow{2}{*}{$\scrR$} & \multirow{2}{*}{$O^{(t,\preA)}$} & \multicolumn{2}{c}{$M^{(t, \preA)}$} & \multirow{2}{*}{$A^{(t)}$} & \multirow{2}{*}{$O^{(t, \postA)}$} & \multicolumn{2}{c}{$M^{(t, \postA)}$} \\
    \cmidrule{3-4} \cmidrule{7-8}
    & & sending & receiving & & & sending & receiving \\
    \midrule
    $\msfM$ & $\emptyset$ & $\emptyset$ & $\emptyset$ & $\WS,\TECH$ & $\EMS$ & $\WS, \FP \; (\text{to } \msfR)$ & $\QUT \;(\text{from } \msfR)$ \\
     \\ 
    \multirow{2}{*}{$\msfR$} & \multirow{2}{*}{$\emptyset$} & \multirow{2}{*}{$\emptyset$} & \multirow{2}{*}{$\WS, \FP \; (\text{from } \msfM)$} & \multirow{2}{*}{$\MKT,\RT$} & \multirow{2}{*}{$\emptyset$} & $\AD, \RT \; (\text{to } \msfC)$ & \multirow{2}{*}{$\QUT \; (\text{from } \msfC)$} \\
    & & & & & &  $\QUT \; (\text{to } \msfM)$  & \\
    \\
    $\msfC$ & $\emptyset$ & $\emptyset$ & $\AD, \RT \; (\text{from } \msfR)$ & $\WTP,\QUT$ & $\emptyset$ & $\QUT \; (\text{to } \msfR)$ & $\emptyset$ \\
    \bottomrule
\end{tabular}
}
{\emph{Note.} $\WS$: wholesale price;  $\TECH$: level of investment in low-carbon technology; $\MKT$: marketing budget; $\RT$: retail price; $\WTP$: willingness to pay; $\QUT$: purchase quantities; $\EMS$: carbon emissions; $\FP$: carbon footprint; $\AD$: advertisement narratives. See 
Example~\ref{ex:scm}  and Figure~\ref{fig:LLM-MAS-SCM} for details. 

} 
\end{table}

\subsection{Prompting LLM-Powered Agents: From Numbers to Text and Back} \label{sec:prompt}

The system design optimization in \eqref{eq:so} is numerical: the design $\theta$ is numerical, and the components of the system state $\xi$ are typically numerical. 
In contrast, LLMs are text-native.
They process inputs and produce outputs as sequences of tokens. 
A distinctive feature of LLM-MAS, compared to traditional simulation, is its \emph{semantic interface}.
This interface connects numerical data to text-based LLM queries through carefully constructed prompts, which facilitate embedding numbers into text and  extracting them back out.

The data flow for simulating an LLM-powered agent's action is 
\begin{align*}
\big(\underbrace{\theta}_{\text{design}}, \underbrace{H_i^{(t-1)}, O_i^{(t, \preA)}, M_i^{(t,\preA)}}_{\text{contextual information}}\big) {} \overset{\text{embedding}}{\longrightarrow} {} \underbrace{\scrX_i^{(t)} \to \LLM \to \scrY_i^{(t)}}_{\text{prompting}} {}  \overset{\text{extraction}}{\longrightarrow} {} \underbrace{A_i^{(t)}}_{\text{action}}.
\end{align*}
Here, the design and contextual information are embedded into the prompt $\scrX_i^{(t)}$. 
Querying the LLM with this prompt produces the textual output $\scrY_i^{(t)}$, from which the agent's numerical action $A_i^{(t)}$ is extracted.

Both the embedding\footnote{Here, ``embedding'' refers to filling the textual prompt templates by replacing semantic variables (e.g., $\SUBSIDY$) with specific numerical values. This differs from token embeddings, which are vector representations in high-dimensional spaces learned by LLMs during pre-training.} and extraction steps rely on a prompt template (Figure~\ref{fig:template}). 
The figure's left panel presents a generic template for prompting an LLM. 
It serves as a reference and can be adapted if alternative wording better elicits human-like behavior in a particular application.
The right panel shows a concrete instance of the consumer agent from the supply chain example. 
The template has five primitive blocks, shown in different colors: the agent's role (yellow), its attributes (green), the system design (red), the contextual information (gray), and the output requirements (blue). 
Once defined for a specific agent, with role $\scrR_i$ (e.g., ``consumer'') and attributes $\chi_i$ (e.g., ``eco-aware''), the template remains \emph{fixed} throughout LLM-MAS.

\begin{figure}[ht]
    \FIGURE{
    $
    \begin{array}{c}
    \includegraphics[width=0.8\textwidth]{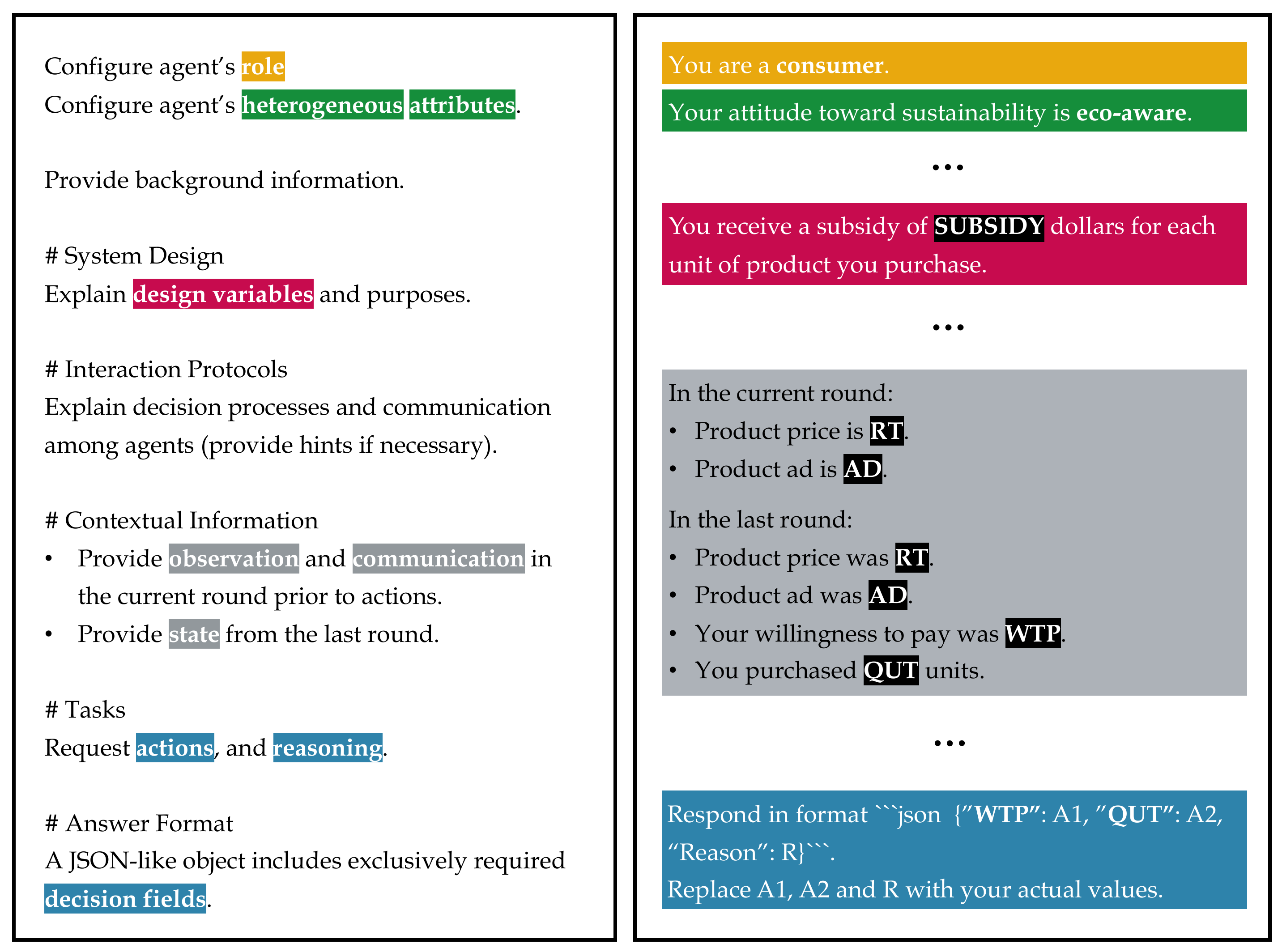} \\
    \end{array}
    $
    }
    {Prompt Template \label{fig:template}
    }
    {Left: General structure, highlighting the key information to be embedded. 
    Right:  Concrete instance for the consumer agent in the supply chain example. Ellipses indicate additional details that help the agent interpret the highlighted content.
    } 
\end{figure}

At run time, only the content inserted in the designated placeholders (black) varies. 
The content includes 
the design $\theta$ (e.g., $\SUBSIDY$) and contextual information $(H_i^{(t-1)}, O_i^{(t, \preA)}, M_i^{(t,\preA)})$ relevant to the agent's action.
For the consumer agent, no pre-action observation is available ($O_i^{(t, \preA)} = \emptyset$), the pre-action communication includes the retail price and the ad narratives communicated from the retailer ($M_i^{(t, \preA)} = (\RT, \AD)$). The carried-over information $H_i^{(t-1)}$ includes $\RT$, $\AD$, willingness to pay ($\WTP$), and purchasing quantity ($\QUT$) from the previous round. See Table~\ref{tab:context-info-scm}.

To simulate an agent's behavior, we prompt the LLM to generate specific actions and prescribe the format in the prompt. A common choice is the standard \texttt{JSON} format (e.g., \verb|```|\texttt{json \{"WTP": A1, "QUT": A2, "Reason": R\}}\verb|```|). 
The desired output includes the action fields $A_i^{(t)}$ (e.g., \texttt{"WTP"} and \texttt{"QUT"}) and an auxiliary field (e.g., \texttt{"Reason"}) that provides a lightweight chain-of-thought reasoning for the chosen actions.
We then parse the output string $\scrY_i^{(t)}$ to extract $\texttt{A1}$ and $\texttt{A2}$ and convert them back to numbers.

\begin{remark}
Our system design optimization via LLM-MAS differs fundamentally from prompt optimization \citep{Liu23pretrain}, which seeks to improve LLM performance through linguistic refinements to the prompt, such as revising structure, rephrasing descriptions, or adding logical hints. In our setting, the prompt text is held fixed; optimization operates solely over the numerical quantities embedded in the prompt.
\end{remark}

\section{Algorithm}  \label{sec:algo}

A standard approach for solving the stochastic optimization problem \eqref{eq:so} is to use first-order methods.
For example, we can iteratively update the estimate of the optimal solution through
\begin{align}
    \theta_{k+1} = \Pi_\Theta(\theta_{k} - \eta_k \nabla f(\theta_k)), \label{eq:sa}
\end{align}
where $\theta_{k}$ is the iterate at step $k$, $\eta_k > 0$ is the stepsize, and $\Pi_\Theta$ denotes the Euclidean projection onto the feasible set $\Theta$.
However, applying this method in the setting of LLM-MAS poses distinctive challenges. 
In the following, we first describe these challenges, then develop an algorithm to address them, and finally discuss variance reduction techniques to enhance the algorithm's efficiency.

\subsection{Technical Challenges} \label{sec:challenges}

The first challenge is gradient computation. 
Because the uncertainty is decision dependent, we cannot calculate the gradient of $ f(\theta) = \E_{\xi \sim \mu_\theta}[F(\theta;\xi)]$ by differentiating $F(\theta; \xi)$ alone. 
This issue is commonly resolved using a change of measure. Consider a probability distribution $\kappa$ on $\Theta$ that is independent of $\theta$ and dominates $\mu_\theta$ (i.e., $\mu_\theta$ is absolutely continuous with respect to $\kappa$). 
Under regularity conditions that permit interchanging differentiation and expectation, one can show  that 
\begin{align}
    \nabla f(\theta) = \E_{\xi\sim \kappa}\Bigl[ \underbrace{\nabla_\theta F(\theta; \xi)L(\theta;\xi) + F(\theta;\xi) \nabla_\theta L(\theta;\xi) }_{\coloneqq h(\theta;\xi)}\Bigr],
    \label{eq:grad}
\end{align}
where $L(\theta; \xi) = (\dd \mu_\theta/\dd \kappa)(\xi)$ is the likelihood ratio (i.e., the Radon--Nikodym derivative) of $\mu_\theta$ with respect to $\kappa$; see \cite{LEcuyer90}.
When $\eqref{eq:grad}$ holds, $h(\theta;\xi)$ is an unbiased estimator of $\nabla f(\theta)$.
Plugging it into \eqref{eq:sa} yields the stochastic gradient descent algorithm. 

In LLM-MAS,
however, $h(\theta;\xi)$ is intractable.
Even though we may choose any $\theta$-independent distribution $\kappa$, 
the likelihood ratio $L(\theta;\xi)$ is generally not computable because  
$\mu_\theta$ is highly complex. 
Given $\theta$, the generation of $\xi$ is autoregressive, multi-round, and multi‑agent:
each token depends on the full prior context, messages pass among agents over successive rounds, so small early random changes can cascade into large downstream shifts, creating deep dependencies. 
As noted in Remark~\ref{remark:uncertainty}, 
additional randomness due to mixed‑precision arithmetic, GPU parallelism and batching across concurrent requests by LLM service providers further increases the complexity of $\mu_\theta$, rendering $L(\theta;\xi)$ practically intractable to evaluate.
We address this challenge with a zeroth‑order approach that estimates $\nabla f(\theta)$ from noisy function evaluations of $f(\theta)$ 
(with a controlled bias), 
avoiding the need to know or learn $\mu_\theta$.

The second challenge concerns steady-state performance. 
We cannot directly draw independent samples from $\mu_\theta$, the stationary distribution of the controlled Markov chain.
Instead, we generate a long trajectory and use the terminal state $\xi^{(T)}$ to approximate $\mu_\theta$ when $T$ is large enough.
LLM queries are costly because inference time scales with the number of tokens processed: the model first encodes the input tokens (prefill stage), then generates the output tokens one at a time (decode stage), conditioning on all previously seen tokens due to its autoregressive nature, and at each step each token goes through the deep transformer architecture. Consequently, long prompts and responses incur proportionally high computational cost.
In a multi‑agent simulation, each round may need several queries per agent to generate actions and exchange messages, so total processed tokens grow with both horizon and agent count.  
Therefore, simulating a sufficiently long trajectory is token-intensive and suffers from high latency. 
In the supply chain example, simulating a single $100$-round trajectory takes roughly half an hour on a local server running a small model (LLaMA3.1-8B). 

The high cost of generating a long trajectory makes it prohibitive to use independent samples of $\xi^{(T)}$ for each iterate $\theta_k$ in the recursion \eqref{eq:sa}. 
Instead of simulating many trajectories, we propose an algorithm that uses one single trajectory across optimization iterations to achieve substantial savings in LLM query cost.

\subsection{Zeroth-Order Approach} \label{sec:zosa}

For simplicity, assume for now that, given $\theta$, we can sample $\xi$ from the stationary distribution $\mu_\theta$ (we relax this later).
We estimate $\nabla f(\theta)$ by evaluating $F(\theta;\xi)$ at small perturbations of $\theta$, with each $\xi$ generated from the corresponding perturbed version of $\mu_\theta$.
Many zeroth‑order methods follow this framework; see \cite{PrashanthShalabh25}. 
We adopt a gradient estimator based on randomized, simultaneous perturbations \citep{NesterovSpokoiny17}, which is popular for its ease of use and sample efficiency in multi‑dimensional settings.

Consider a smoothed version of the objective function $f(\theta)$:
\begin{align}
    f_{\delta}(\theta) \coloneqq \E_{u \sim \msfLambda}[f(\theta+\delta u)],  \label{eq:smoothed-f}
\end{align}
where  $\delta > 0$ is the smoothing parameter (i.e., perturbation radius) and $u \in \RR^{d}$ is a zero-mean random direction drawn from a $d$-dimensional distribution $\msfLambda$. 
When $\delta$ is small, $f_\delta(\theta)$ and its gradient $\nabla f_\delta(\theta)$ closely approximate $f(\theta)$ and $\nabla f(\theta)$. 
We take $\msfLambda=N(0,d^{-1}I_d)$, the standard multivariate Gaussian distribution normalized by the dimension $d$, which ensures the perturbations satisfy $ \E_{u\sim \msfLambda}[\|u\|^2] = 1$. 
By Stein's identity, 
\begin{align}
\nabla f_{\delta} (\theta) = \frac{d}{\delta}\E_{u \sim \msfLambda}[f(\theta+\delta u)u], \label{eq:stein}    
\end{align}
which leads to the unbiased estimator of $\nabla f_\delta(\delta)$: 
\begin{align}
    G(\theta;\xi^{\pm},u,\delta) \coloneqq \frac{d}{2\delta} \left(F(\theta+\delta u; \xi^+) - F(\theta-\delta u; \xi^-) \right) u, \quad u \sim N(0,d^{-1}I_d),\label{eq:gradest-two} 
\end{align}
where $\xi^{\pm} = (\xi^+, \xi^-)$, drawn from the perturbed distributions $\mu_{\theta+\delta u}$ and $\mu_{\theta-\delta u}$, respectively.

This estimator uses two samples of $\xi$ from a pair of symmetric random perturbations of $\theta$, which is query-efficient for LLM-MAS.
In contrast, the classic central finite-difference estimator for $\nabla f(\theta)$ applies coordinate-wise perturbations and needs two evaluations per dimension, resulting in a total of $\scrO(d)$ state samples.
This is often prohibitively costly in LLM-MAS when $\theta$ is multi-dimensional. 
Moreover, deterministic, coordinate-wise perturbations can under‑explore the parameter space compared with randomized directions, potentially causing stagnation when optimizing nonconvex objectives \citep{Chin97}.

\subsection{On-Trajectory Learning} \label{sec:otl}

To use the zeroth-order estimator \eqref{eq:gradest-two} in the recursion \eqref{eq:sa}, 
we face one main difficulty: we cannot  sample directly from $\mu_\theta$. 
For any given $\theta$, 
the stationary distribution of the  Markov chain $\Psi_\theta$ can only be approximated through  sufficiently long simulations of $\xi^{(t)}$. 

A straightforward implementation then proceeds as follows. 
At each iterate $\theta_k$, 
we simulate two trajectories, corresponding to the perturbations $\theta_k+\delta_k u_k$ and $\theta_k-\delta_k u_k$, where $\delta_k$ is the smoothing parameter and $u_k$ is the random perturbation direction.
The terminal state of each trajectory approximates a sample from the corresponding perturbed stationary distribution. 
We refer to this approach as the \emph{multi-trajectory learning} (MTL) algorithm (Figure~\ref{fig:zosa}), because it requires many trajectories throughout the optimization.

However, generating long trajectories between successive iterations is costly in LLM-MAS. 
A single round of simulation can take tens of seconds or even minutes, and the runtime grows with the number of agents 
and communication links. 
Waiting for LLM‑MAS to finish a long trajectory before updating design parameters is therefore excessively expensive. 
In addition, only the terminal state $\xi^{(T)}$ contributes to gradient estimation, while intermediate states $\xi^{(1)}, \ldots, \xi^{(T-1)}$ are  discarded even though they potentially contain valuable information.
This wastes LLM queries since each state transition consumes tokens, leading to higher token usage and longer runtime with limited gain in information per query.

To address the sample inefficiency of MTL, which stems from decoupling parameter updates from gradient estimation, we propose an \emph{on-trajectory learning} (OTL) algorithm (Figure~\ref{fig:zosa}). 
OTL operates on one single trajectory, along which it updates the design parameters and concurrently forms the zeroth‑order gradient estimate using a pair of one‑round simulations initialized from the same current state but executed under two perturbed designs. 
Specifically, at each iteration $k$,
\begin{align}
    \makebox[4.5cm][l]{(Random Perturbation)} &
    \begin{cases} 
    u_k \sim N(0, d^{-1}I_d), \\[0.6ex]
    \xi_k^{+} = \Psi(\xi_k \vert \theta_k + \delta_k u_k), \\[0.6ex]
    \xi_k^{-} = \Psi(\xi_k \vert \theta_k - \delta_k u_k),
    \end{cases}
    \label{eq:perturb-step} \\[0.5ex]
    \makebox[4.5cm][l]{(Gradient Estimation)} &
    G_k = G(\theta_k; \xi_k^{\pm}, u_k, \delta_k),
    \label{eq:grad-step} \\[0.5ex]
    \makebox[4.5cm][l]{(Parameter Update)} &
    \theta_{k+1} = \Pi_{\Theta}(\theta_k - \eta_k G_k),
    \label{eq:zosa} \\[0.5ex]
    \makebox[4.5cm][l]{(State Update)} &
    \xi_{k+1} = \Psi(\xi_k \vert \theta_k),
    \label{eq:state-update-step}
\end{align}
where we write $\Psi_\theta(\xi)$ as $\Psi(\xi|\theta)$ for notational convenience.

The OTL trajectory does not come from the Markov chain $\Psi_\theta$ with a fixed $\theta$; 
instead, each state is generated from the previous one under an adaptively changing $\theta$ in \eqref{eq:state-update-step}. 
The samples $\xi^{\pm}_k = (\xi_k^+, \xi_k^-)$ from random perturbations in \eqref{eq:perturb-step} are generated from the transition distribution of the Markov chain with the perturbed design parameter, rather than from the corresponding stationary distributions assumed by the zeroth-order estimator \eqref{eq:gradest-two}.
This mismatch introduces 
bias in estimating $\nabla f_{\delta_k}(\theta_k)$, the smoothed counterpart of the gradient $\nabla f(\theta_k)$.

To prevent the bias from hindering convergence of the iterates $\theta_k$, we adopt a \emph{two-timescale} scheme. The gradient estimate $G_k$ evolves on the fast timescale so it can track the targets induced by the current design parameter.
The design parameter $\theta_k$ evolves on the slow timescale, updating only after the fast process has effectively stabilized near the regime defined by $\theta_k$. 
This two-timescale scheme is implemented by selecting the perturbation radius $\delta_k$ for gradient estimation and the stepsize $\eta_k$ for parameter updating such that both diminish with iterates while $\eta_k/\delta_k \to 0$ as $k\to\infty$. 
A common choice is $\delta_k = \delta_0/ (1+k)^\alpha$ and $\eta_k = \eta_0/ (1+k)^\beta$, where $\delta_0, \eta_0 > 0$ and $ \alpha < \beta$; see Assumption~\ref{ass:stepsize} in the Appendix for details.

\begin{figure}[ht]
    \FIGURE{
\begin{minipage}[t]{0.50\textwidth} \vspace{-\baselineskip}
\begin{algorithm}[H]
\captionof{algorithm}{On-Trajectory Learning} 
\begin{algorithmic}[1]
\renewcommand{\algorithmicrequire}{\textbf{Input:}}
\renewcommand{\algorithmicensure}{\textbf{Output:}}
\Require $\delta_0, \eta_0, \alpha, \beta, \xi_0, \theta_0$
\Ensure Last iterate $\theta_k$
\For{$k=0,1,\ldots$}
\Comment{LLM-MAS}
    \State Set $\delta_k = \delta_0/(1+k)^{\alpha}$, $\eta_k = \eta_0/(1+k)^{\beta}$
    \State Draw $u_k \sim N(0,d^{-1}I_d)$
    \newline
    \State Generate $\xi^+_k = \Psi(\xi_{k} \vert \theta_k + \delta_k u_k)$ 
    \State Generate $\xi^-_k = \Psi(\xi_{k} \vert \theta_k - \delta_k u_k)$ 
    \newline
    \State Compute $G_k = G(\theta_k;(\xi_k^+, \xi_k^-),u_k,\delta_k)$
    \State Set $\theta_{k+1} = \Pi_{\Theta}(\theta_k - \eta_k G_k)$
    \State Generate $\xi_{k+1} = \Psi(\xi_k \vert \theta_{k})$ 
\vspace{0.7ex}
\EndFor
\end{algorithmic}
\end{algorithm}
\end{minipage} \hfill
\begin{minipage}[t]{0.50\textwidth} \vspace{-\baselineskip}
\begin{algorithm}[H]
\captionof{algorithm}{Multi-Trajectory Learning} 
\begin{algorithmic}[1]
\renewcommand{\algorithmicrequire}{\textbf{Input:}}
\renewcommand{\algorithmicensure}{\textbf{Output:}}
\Require $\delta_0, \eta_0, \alpha, \beta, T, \xi_0^{(0)}, \theta_0$
\Ensure Last iterate $\theta_k$
\For{$k=0,1,\ldots$}
    \State Set $\delta_k = \delta_0/(1+k)^{\alpha}$, $\eta_k = \eta_0/(1+k)^{\beta}$
    \State Draw $u_k \sim N(0,d^{-1}I_d)$
    \For{$t=0, \ldots, T-1$} \Comment{LLM-MAS}
        \State Generate $\xi_k^{+,(t+1)}  = \Psi(\xi_k^{+, (t)}|{\theta_k + \delta_k u_k})$ 
        \State Generate $\xi_k^{-,(t+1)}  = \Psi(\xi_k^{-, (t)}|{\theta_k - \delta_k u_k})$ 
    \EndFor
    \State Compute $G_k = G(\theta_k;(\xi_k^{+,(T)},\xi_k^{-,(T)}), u_k,\delta_k)$    
    \State Set $\theta_{k+1} = \Pi_{\Theta}(\theta_k - \eta_k G_k)$
    \newline
\EndFor
\end{algorithmic}
\end{algorithm}
\end{minipage}
    }
    {On‑Trajectory vs. Multi‑Trajectory Learning  \label{fig:zosa}
    }
    {\begin{enumerate*}[label=(\roman*)]
        \item For on-trajectory learning, the round index $t$ used in LLM-MAS is dropped because it coincides with the algorithm iterate index $k$. 
        \item For multi-trajectory learning, in the inner loop over $t$, $\xi_k^{+, (0)} = \xi_k^{-, (0)} = \xi_0^{(0)}$ for all $k$. 
    \end{enumerate*}} 
\end{figure}

\subsection{Asymptotic Convergence} \label{sec:asymp}

As the objective function $f(\theta)$ in~\eqref{eq:so} may be nonconvex, 
we target convergence to stationary points. 

\begin{theorem} \label{thm:consistency}
    Under Assumptions \ref{ass:feasible-set}--\ref{ass:stepsize} in the Appendix, the OTL iterates $\{\theta_k:k\geq 0\}$ converge almost surely to the set of stationary points $\Theta_0 = \left\{\theta\in\Theta: 0 \in \nabla f(\theta) + \scrN_{\Theta}(\theta)\right\}$, where $\scrN_\Theta(\theta) \coloneqq \{x\in\RR^d: \langle x, \vartheta-\theta \rangle \le 0, \forall \vartheta\in\Theta\}$ denotes the normal cone of $\Theta$ at $\theta$.
\end{theorem}

We analyze the OTL algorithm using the ordinary differential equation (ODE) approach for SA \citep{KushnerYin03}. 
We treat the recursion in \eqref{eq:perturb-step}--\eqref{eq:state-update-step} as a discretization of a limiting ODE and show that the iterates asymptotically track it, which yields convergence to a stationary point of problem \eqref{eq:so}. 

Applying this classic approach to our setting raises technical complications due to intertemporal dependencies from adaptive Markovian data. 
Recent work on SA with such data typically assumes access to gradients, a key difference from our setting. 
For example, \cite{CheDongTong26} give a finite-time analysis of stochastic gradient descent with adaptive data based on Lyapunov functions and identify conditions under which the convergence rate approaches that of the independent-noise case for objectives with structure such as convexity. 
\cite{LiLiangChenZhang26} further improve the rate 
using the Poisson equation approach.
Our setting, which applies zeroth‑order gradient estimation to steady-state stochastic optimization, is closest to \cite{LEcuyerGlynn94}, but their analysis is limited to single-server queues and relies on the queueing structure.

We also follow the Poisson equation approach. 
The key in our proof is to control the bias that the zeroth‑order gradient estimator \eqref{eq:gradest-two} accumulates along a trajectory.
First, we address the nonstationary bias from the coupling of parameter updates and gradient estimates in steps \eqref{eq:grad-step}–\eqref{eq:state-update-step} using the two‑timescale scheme in Assumption \ref{ass:stepsize}.
This ensures that the transient gradient $H(\theta;\xi)$ (see Assumption~\ref{ass:abs-conti}) consistently aligns with the target $\nabla f(\theta)$. 
Second, we control the bias resulting from the zeroth-order approach in step \eqref{eq:perturb-step} by exploiting the smoothness of the transition kernel under Assumption~\ref{ass:trans-kern}.
This yields a bound on the gap between the gradient estimator \eqref{eq:gradest-two} and $H(\theta;\xi)$.
See Section~\ref{ec:proof} of the e-companion for details.

\section{Efficiency Enhancement: Variance Reduction} \label{sec:enhance}

Gradient estimation in Markov chains becomes increasingly challenging as the trajectory length grows because the variance of such estimators often rises sharply with time horizon \citep[Chapter VII]{AsmussenGlynn07}. 
In LLM‑MAS, this difficulty is further amplified by multi‑agent interactions within each simulation round. As agents adapt to one another, small local changes can spread through the system and lead to pronounced variability, making variance reduction 
crucial for efficient use of LLM queries.

In the following, we introduce two variance reduction techniques from the zeroth-order optimization literature and demonstrate how to incorporate them in LLM-MAS. 
They are complementary. 
The first applies when $F$ depends on $\theta$ explicitly (i.e., $\nabla_\theta F(\theta;\xi) \neq 0$); it leverages the explicit coupling between the design parameters and system performance. 
The second has broader scope and covers cases where $F$ depends on $\theta$ only through $\xi$ (i.e., $\nabla_\theta F(\theta;\xi) = 0$); 
it instead uses the information accumulated during optimization to handle settings where the coupling between $\theta$ and $F$ is implicit.

\subsection{Guided Perturbation} \label{sec:gp}

In the gradient estimator \eqref{eq:gradest-two},
the perturbation vector $u$ is drawn from an isotropic Gaussian, which explores all directions equally, whereas ideally one  would follow the direction of the unknown gradient $\nabla f(\theta)$.
The \emph{guided perturbation} (GP) method, proposed by \cite{Maheswaranathan19} for settings without noise or with independent noise, 
addresses this by using a \emph{surrogate gradient} $\phi$ to direct the perturbations toward more promising regions. 

Specifically, the GP method samples $u$ from a non-isotropic Gaussian $N(0, \Sigma)$, where the covariance matrix $\Sigma$ is a weighted combination of an isotropic component and a gradient‑aligned component: 
\begin{align}
    \Sigma \coloneqq w \frac{I_d}{d} + (1-w) \frac{\phi \phi^\intercal}{\|\phi\|^2}. \label{eq:cov}
\end{align}
It stretches the spherical Gaussian into an ellipsoid elongated along the surrogate gradient direction. This structure features a distinct eigenvalue $w/d + (1-w)$ in the direction of $\phi$, and $d-1$ identical eigenvalues $w/d$ along directions orthogonal to it.
The weight $w\in[0,1]$ controls the trade-off between exploration and exploitation: a larger value promotes isotropic exploration, while a smaller value steers perturbations toward the surrogate gradient. 
In the extreme case when $w = 1$, the perturbation distribution reduces to $N(0,d^{-1}I_d)$ as in the basic gradient estimator. As $w$ decreases, the perturbations become increasingly aligned with $\phi$, concentrating the search along that direction while reducing the variability in orthogonal directions.

The key to applying this method is the choice of $\phi$: it is problem-specific and should align closely with $\nabla f(\theta)$. 
In our setting, when $\theta$ enters $F$ explicitly, $\nabla_\theta F(\theta;\xi)$ in~\eqref{eq:grad} is nonzero.
Although $\nabla_\theta F(\theta;\xi)$ is not an unbiased estimator of $\nabla f(\theta)$ due to the decision-dependent uncertainty, 
it remains closely related with the true gradient and is therefore informative. 
We set $\phi = \nabla_\theta F(\theta; \xi)$ to exploit this information in LLM‑MAS.

We denote the GP gradient estimator by 
\begin{align}
    G^{\GP}(\theta;\xi, \xi^\pm, u, \delta) =  \frac{1}{2\delta} \left(F(\theta+\delta u; \xi^+) - F(\theta-\delta u; \xi^-) \right) u, \quad u \sim N(0,\Sigma(\xi)), \label{eq:gradest-gp} 
\end{align}
where we write $\Sigma(\xi)$ to highlight its dependence on $\xi$ through the surrogate gradient $\phi$. 

We incorporate GP into OTL as follows. 
At each iteration $k$, we modify the first two steps \eqref{eq:perturb-step} and \eqref{eq:grad-step} to 
\begin{align}
    \makebox[4.5cm][l]{(Random Perturbation)} &
    \begin{cases} 
    u_k \sim N(0, \Sigma(\xi_k)), \\[0.6ex]
    \xi_k^{+} = \Psi(\xi_k \vert \theta_k + \delta_k u_k), \\[0.6ex]
    \xi_k^{-} = \Psi(\xi_k \vert \theta_k - \delta_k u_k),
    \end{cases}
    \label{eq:perturb-step-gp} \\[0.5ex]
    \makebox[4.5cm][l]{(Gradient Estimation)} &
    G_k = G^{\GP}(\theta_k; \xi_k, \xi_k^{\pm}, u_k, \delta_k),
    \label{eq:grad-step-gp}
\end{align}
while the remaining steps follow those in \eqref{eq:zosa} and \eqref{eq:state-update-step}.

Because $\Sigma$ is positive definite, the expected value of this gradient estimator points in a descent direction.
However, with non-isotropic perturbations drawn from $\msfLambda = N(0, \Sigma)$, $G^{\GP}(\theta;\xi, \xi^\pm, u, \delta)$ is no longer unbiased with respect to the gradient of the smoothed objective 
$f_{\delta}(\theta) = \E_{u \sim \msfLambda}[f(\theta+\delta u)]$ in \eqref{eq:smoothed-f}; see Lemma~\ref{lem:gp-bias} of the e-companion. 
To avoid this bias, arising from the dependence of $\Sigma(\xi)$ on $\xi$, that could  harm the convergence of the OTL algorithm, we vary the weight $w$ across iterations so that the GP gradient estimator asymptotically approaches the basic gradient estimator \eqref{eq:gradest-two}. 
Specifically, we set $w_{k+1} = 1 - \rho(1-w_k)$ for some $\rho\in(0,1)$, then $1-w_k$ decays geometrically with contraction factor $1-\rho$, implying that $w_k \to 1$.
This scheduling enables guided exploration in early iterations while preserving the overall convergence, particularly when the smoothing parameter decays polynomially as $\delta_k = \delta_0/(k+1)^\alpha$.

\subsection{Residual Feedback} \label{sec:rs}

When $\nabla_\theta F(\theta;\xi) = 0$, the system performance depends on $\theta$ only through the distribution of $\xi$, so the GP method does not apply. 
We instead use the \emph{residual feedback} (RF) method, proposed by \cite{Zhang22} for deterministic or stochastic settings with independent noise, and extend it to our setting where the distribution of $\xi$ changes adaptively with the iterate $\theta_k$. 

Unlike the basic gradient estimator \eqref{eq:gradest-two} and the GP gradient estimator \eqref{eq:gradest-gp}, both of which use only samples in the current iteration,
the RF method additionally uses samples from the previous iteration. 
If the noise were independent,  which we denote by $\varepsilon_k$ to distinguish it from $\xi_k$, we would estimate $\nabla f(\theta_k)$, using one sample from iteration $k$ and another from iteration $k-1$, via $d \delta_k^{-1} (F(\theta_k+\delta_k u_k; \varepsilon_k) - F(\theta_{k-1}+\delta_{k-1} u_{k-1}; \varepsilon_{k-1}))u_k$, where the perturbation vectors $u_{k}$ and $u_{k-1}$ are drawn from  $N(0, d^{-1}I_d)$ as in \eqref{eq:gradest-two}. 
The difference between the two function evaluations represents the residual, as it carries information from the most recent evaluation, thus giving rise to the method's name.
Conceptually, it is closely related to the control variate method, which also leverages the correlation between a target random variable and an auxiliary one to reduce variance \citep[Chapter V]{AsmussenGlynn07}.

However, in our setting the ``noise'' $\xi_k$ is Markovian, and its transition distribution changes adaptively with the iterate $\theta_k$. 
Because we also perturb $\theta_k$, we need to specify the parameter value $\theta$ that governs the transition distribution generating each sample. 
In other words, we should decide which random variables are to replace $\varepsilon_k$ and $\varepsilon_{k-1}$ in the RF gradient estimator for the independent-noise case.

We use samples drawn under the perturbed distributions for those roles: 
\begin{align}
    &  G^{\RF}(\theta_k;  u_k, \delta_k, \xi_k^+, \theta_{k-1}, u_{k-1}, \delta_{k-1},  \xi_{k-1}^+) \nonumber \\ 
    \coloneqq & \frac{d}{\delta_k} \left(F(\theta_k + \delta_k u_k; \xi_k^+) - F(\theta_{k-1}+\delta_{k-1}u_{k-1};\xi_{k-1}^+) \right)u_k,  \quad u_k, u_{k-1}\sim N(0, d^{-1}I_d), \label{eq:gradest-rs}
\end{align}
where $\xi_k^+$ and $\xi_{k-1}^+$ are drawn from the perturbed distributions associated with parameters $\theta_k + \delta_k u_k$ and $\theta_{k-1} + \delta_{k-1} u_{k-1}$, respectively. 
Note that the second term in \eqref{eq:gradest-rs} has mean zero; i.e., 
$\E[F(\theta_{k-1}+\delta_{k-1}u_{k-1}; \xi_{k-1}^+)u_k] = 0$, because $u_k$ is independent of the other random variables and has mean zero. 
Thus, $d \delta^{-1}F(\theta_{k-1}+\delta_{k-1}u_{k-1}; \xi_{k-1}^+)u_k$ serves as a control variate that pairs with the first term $d \delta^{-1} F(\theta_k + \delta_k u_k; \xi_k^+)u_k$, which is an unbiased estimator of $\nabla f_{\delta_k}(\theta_k)$ by Stein's identity \eqref{eq:stein}, to reduce its variance.

The RF gradient estimator mirrors the structure of the basic gradient estimator $G(\theta;\xi^{\pm},u,\delta)$ in  \eqref{eq:gradest-two} yet requires only one new function evaluation per iteration, reducing the computational cost (LLM queries) by half.  
Rather than explicitly performing the second function evaluation as in \eqref{eq:gradest-two}, this estimator implicitly reconstructs it using past data. 
For further comparison, consider a variant of the basic estimator $\wtG(\theta;\xi^{+},u,\delta) \coloneqq (d/\delta) F(\theta+\delta u; \xi^+) u$, which incurs the same computational cost as the RF gradient estimator and is also  unbiased for $\nabla f_\delta(\theta)$ when $u\sim N(0, d^{-1}I_d)$. However, this estimator has substantially higher variance because it lacks a control variate.

To incorporate RF into OTL, 
 we modify the first two steps \eqref{eq:perturb-step} and \eqref{eq:grad-step} of each iteration $k$ to 
\begin{align}
    \makebox[4.5cm][l]{(Random Perturbation)} &
    \begin{cases} 
    u_k \sim N(0, d^{-1}I_d), \\[0.6ex]
    \xi_k^{+} = \Psi(\xi_k \vert \theta_k + \delta_k u_k), 
    \end{cases}
    \label{eq:perturb-step-rf} \\[0.5ex]
    \makebox[4.5cm][l]{(Gradient Estimation)} &
    G_k = G^{\RF}(\theta_k;  u_k, \delta_k, \xi_k^+, \theta_{k-1}, u_{k-1}, \delta_{k-1},  \xi_{k-1}^+),
    \label{eq:grad-step-rf}
\end{align}
while the remaining steps follow those in \eqref{eq:zosa} and \eqref{eq:state-update-step}.

When $F$ varies smoothly as the design parameters move from $\theta_{k-1}+\delta_{k-1}u_{k-1}$ in the previous iteration to $\theta_k+\delta_k u_k$ in the current one, the two evaluations are  correlated. 
This correlation allows the second term in $G^{\RF}$ to act as a control variate, 
effectively offsetting part of the Monte Carlo noise in the gradient estimator. 
As a result, the RF method retrieves the information gathered throughout the algorithm's learning process, effectively reducing the variance.

\section{Numerical Experiments: Efficiency Comparison} \label{sec:scm}

We evaluate our algorithm on the supply chain example. 
The policymaker chooses a carbon tax $\theta_1$ and the consumer subsidy $\theta_2$ to minimize an objective $F(\theta;\xi)$ that balances supply chain welfare, fiscal costs, and environmental externalities.
We also incorporate agent-specific behavioral attributes to reflect their engagement in sustainability initiatives.
Additional details are provided in Section~\ref{ec:ex} of the e-companion.

We implement the OTL algorithm and enhance it with the GP method (OTL-GP) since $\nabla_\theta F(\theta;\xi)$ is nonzero in this example. 
We also consider the following benchmarks. 
\begin{enumerate}[label=(\roman*)] 
    \item Bayesian optimization (BO): A global optimization method that treats the objective $f(\theta)$ as a black-box and builds a surrogate for it (typically a Gaussian process) to guide the search for optimal solutions without gradient information \citep{Frazier18}. 
    Because BO requires independent samples of $\xi$ from the steady-state distribution $\mu_\theta$, each iteration in our implementation takes the terminal state $\xi^{(T)}$ of a long trajectory with $T = 100$ to approximate $\mu_\theta$, as in the MTL algorithm. 
    \item LLM-as-a-Solver: An LLM performs numerical optimization without domain context. 
    We only provide the candidate solutions with their objective values from previous iterations, and then use this information in the next prompt to request a new candidate.
    In this way, the LLM refines its search over iterations, similar to a standard numerical optimization algorithm. 
    Like BO, it also uses multiple trajectories.
    \item LLM-as-a-Solver (CoT): Similar to the above, but the prompt integrates the chain-of-thought technique \citep{wei2022chain} to improve  reasoning and guide the model in refining subsequent solutions.
    \item LLM-as-a-Designer: An LLM acts as the system designer and proposes solutions using its internal logic, without being given the stochastic optimization formulation in \eqref{eq:so}.
    The prompt includes the actions of all agents simulated via LLM-MAS and the corresponding system performance for each round. 
    This information guides the LLM to propose improved designs over successive iterations, using a single trajectory as in the OTL algorithm.
\end{enumerate}
When comparing these algorithms, we track how the objective changes with the number of LLM queries, not the number of iterations, because some use one trajectory per iteration while others use a single trajectory across all iterations.
The results are shown in Figure~\ref{fig:res-scm-longrun}.

\begin{figure}[ht]
    \FIGURE{
    \includegraphics[width=0.85\textwidth]{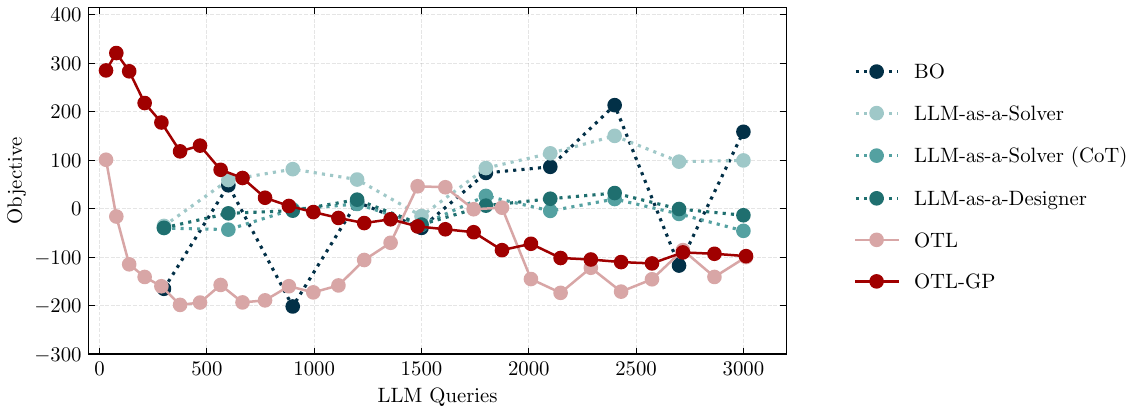}
    }
    {Algorithm Performance \label{fig:res-scm-longrun}
    }
    {\begin{enumerate*}[label=(\roman*)]
        \item LLaMA3.1-8B is the backbone for all agents.
        \item Each one-round simulation operator $\Psi_\theta$ requires three LLM queries.
        The additional query for mimicking the retailer agent's tool-call (for ad-narrative generation) is not included in the count.
    \end{enumerate*}}
\end{figure}

OTL reduces the objective quickly within the first $500$ LLM queries. It then rebounds, declines again, and stabilizes around $2000$ queries. This pattern indicates high sensitivity to noise in LLM-MAS, even though OTL eventually converges. In contrast, OTL-GP benefits from the variance reduction technique and shows a steady decrease without significant oscillations. The curve also drops early, though more slowly than OTL, and it improves the solution consistently over time.

BO updates less frequently than OTL and OTL-GP because it relies on an independent long trajectory for each sample. 
Within the first $500$ LLM queries, it sometimes reaches low objective values, but the curve then fluctuates widely and stays unstable even after $2000$ queries, with large values that suggest ongoing exploration.
This behavior arises mainly because the high query cost of independent samples limits the number of iterations, which slows the surrogate model in learning the objective function and delays exploitation, especially in the presence of heavy noise.
Exploration is critical for the convergence of BO to a globally optimal solution, but in this setting with limited query budget and noisy samples, much of the benefit is lost, leading to unsatisfactory performance.

When used purely as a numerical solver (LLM-as-a-Solver) without domain context, the LLM performs poorly, and incorporating CoT reasoning offers only marginal gains. 
The LLM acts as a language-based learner, drawing on previous observations to generate new candidate solutions. 
Unlike BO, it neither constructs a surrogate of the objective nor applies an acquisition function to balance exploration and exploitation. 
Instead, it relies on heuristics and infers patterns from text without a principled treatment of uncertainty. 
It also lacks a metric to quantify the effect of incremental updates across iterations, thereby limiting its ability to systematically explore the solution space, resulting in substantial performance degradation.

Assigning the LLM to role-play a policymaker (LLM-as-a-Designer) and providing sufficient application background still does not yield notable gains.
The prompt enriches domain understanding but does not endow the model with a principled search method, so LLM can only explore heuristically.  
Because it operates on a single trajectory where the temporal dependence among system states is hard to untangle,
it becomes even more difficult for the LLM to tell whether a change in the solution improves or hurts the performance.

These results underscore the need to frame system design as a numerical optimization problem, rather than relying on plain-language  descriptions of the problem and directly prompting an LLM for decisions. They also align with \citet{Simchi26}, who argue that LLMs should complement, not substitute, mathematically grounded models for optimizing complex stochastic systems. In line with this principle, our framework is best viewed as a human-in-the-loop approach: a human serves as the system designer, while LLM-MAS acts as an assistive component within the system.

\section{Case Study: Innovation Contest Design} \label{sec:case}

We apply LLM-MAS to optimal contest design. 
Innovation contests are open calls for solutions to complex problems. Individuals or teams compete to submit the best ideas, models, or prototypes. In the digital economy, contests harness global expertise at scale, cut experimentation costs, and speed up discovery by turning hard tasks into  merit-based competitions \citep{Hu21,NittalaErat26}.
Platforms like InnoCentive, TopCoder, and Kaggle, along with landmark competitions such as the Netflix Prize, show how contests mobilize talent, produce performance gains, and spread advanced methods across industries.
For broader discussion of innovation platforms; see \citet{Chen20}.

Consider a contest designer who seeks a mechanism that incentivizes contestants to exert as much effort as possible with the objective of maximizing expected total effort. 
We focus on a specific mechanism: a modified all-pay auction with \emph{entry fee} and \emph{reserve}, which has drawn considerable interest in the contest design literature. 
Under this mechanism, the designer sets a prize budget $V$.
All contestants independently decide whether to enter the contest by paying an entry fee $E$ and, conditional on entry, how much effort to exert. 
If the highest effort exceeds a threshold (i.e., the reserve $\hat{e}$),
the top performer receives the prize $V$ and all entry fees, whereas the others receive nothing. 
Otherwise, if all participants' efforts fall below the reserve, each  receives the same amount of a \emph{shared prize} $S$ that equals the entry fee plus a portion of $V$.  

A salient feature of this mechanism is that the shared prize creates a chance of positive payoffs, which encourages low-ability contestants,
who incur higher costs for any given effort level, to participate. 
At the same time, their entry fees are effectively transferred to the winner, raising the top performer's payoff beyond the base prize budget and thereby strengthening high-ability contestants' incentives to exert effort.

\subsection{Analytical Model}\label{sec:analytical-designs}

This mechanism is characterized by a three-element tuple of design parameters $(E, \hat{e}, S)$. 
Under standard assumptions---including risk neutrality, full rationality, and independently and identically distributed abilities---as well as a common, publicly known \emph{liability} $K$ that caps each contestant's entry fee, 
\cite{Liu18} use a game-theoretic analysis to derive in closed form the design parameters that maximize the expected total effort.\footnote{\cite{HuangZhangZhang26} further show that the same contest structure also maximizes the expected winner's effort, though the optimal parameter values may differ from those that maximize expected total effort. 
We focus on the total-effort objective, but an analysis of the winner-effort objective can be done in the same way.}  
The optimal values are expressed as functions of $K$. 
Specifically, 
\begin{align}
    \left( E, \hat{e}, S\right) = \left( K, \, t^*(K)\left[(V+NK)F^{N-1}(t^*(K)) - K \right], \, \frac{K}{F^{N-1}(t^*(K))} \right), \label{eq:game-theo-optimal}
\end{align}
where $N$ is the number of contestants, $F$ is the cumulative distribution function of a contestant's ability on a finite support $[a,b]$. 
The term $t^*(K)$ denotes the cutoff separating low- and high-ability contestants,  given by
\[t^*(K) = \max\left\{J^{-1}(t_0), \, F^{-1} \left( \left(\frac{N K}{V + N K}\right)^{\frac{1}{N-1}} \right)\right\},\]
where $J(t) = t - [1-F(t)]/f(t)$ is the virtual ability function 
with $f$ being the density associated with $F$. Here, $J^{-1}$ and $F^{-1}$ are the inverse functions of $J$ and $F$, respectively, and $1/t_0$ is the marginal benefit of effort for the contest designer. 
With the optimal design, the maximum expected total effort is 
\begin{align}\label{eq:game-theo-total-effort}
    R^* = \int_{t^*(K)}^b \sum_{i=1}^N (J(t_i) - t_0) [ (V+NK) F^{N-1}(t_i) - K ] f(t_i) \, \dd t_i + t_0 V.
\end{align}
See Figure~\ref{fig:contest-theo} for how the optimal design parameters and maximum total effort vary with $K$.

\begin{figure}[ht]
    \FIGURE{
    $
    \begin{array}{c}
    \includegraphics[width=0.8\textwidth]{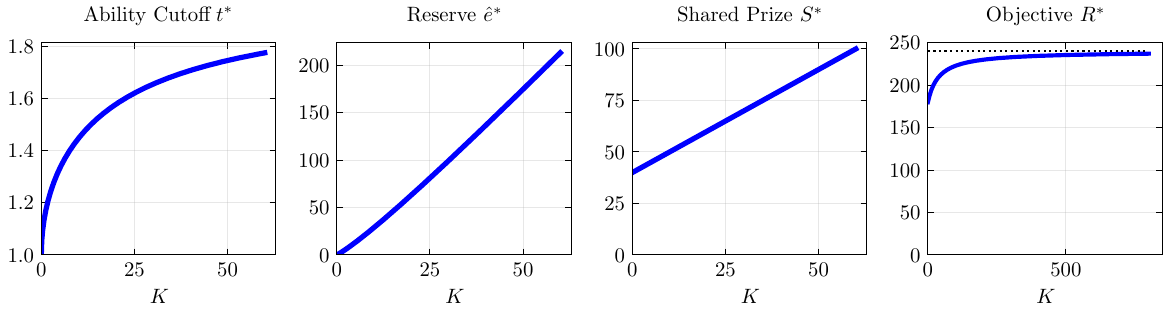}  
    \end{array}
    $
    }
    {Optimal Contest Design (Analytical Model)  \label{fig:contest-theo}
    }
    {Following \cite{HuangZhangZhang26}, we set $N = 3$, $V=120$ and $t_0 = 0$, with $t_i$ drawn from $\Unif[1,2]$ for all $i$. The limiting value of $R^*$ is $bV = 240$, extracting all surplus except that of the highest-ability contestant, which occurs with probability zero. 
    }
\end{figure}

\paragraph{Optimal Prize Allocation.} If no contestant's effort exceeds $\hat{e}^*$, each receives the shared prize $S^*$. Otherwise, the contestant exerting the highest effort receives $V + NK$,  which is the base prize plus the total entry fees, while all others receive $0$. 
Since each participant pays the entry fee $E^*=K$, his net payoff becomes $-K$ (a ``negative prize'') whenever the highest effort exceeds $\hat{e}^*$ but he does not win.

\paragraph{Individual Behavior.} 
Under this optimal design, all contestants choose to enter irrespective of their ability. Their exerted efforts, however, differ by ability.
Low-ability contestants (those with ability below $t^*(K)$) 
exert zero effort, whereas a contestant with ability $t_i > t^*(K)$
exerts effort 
\[e_i = t_i ((V + N K)F^{N-1}(t_i) - K) - \int_{t^*(K)}^{t_i} \bigl((V + N K)F^{N-1}(s) - K\bigr) \, \dd s, \] 
which increases with his ability (Figure~\ref{fig:contest-repl-ind}).

\subsection{Behavioral Data from Lab Experiments} \label{sec:case-lab}

\cite{HuangZhangZhang26} conducted lab experiments to evaluate the game-theoretic optimal designs described in Section~\ref{sec:analytical-designs}. 
They tested six contest designs that vary in the liability $K$ and in the design objective; see Table~\ref{tab:designs}. 
The sessions took place in March 2022 and April 2024 and involved 432 subjects recruited from a university.  
Participants were randomly assigned to one of the six designs, yielding 72 subjects per design.
Under each design, the 72 subjects were grouped into 24 contests, with $N=3$ contestants competing in each.

\begin{table}[ht]
\TABLE
{Contest Designs in \cite{HuangZhangZhang26} \label{tab:designs}}
{
\begin{tabular}{@{\extracolsep{20pt}} c@{\extracolsep{20pt} }  c@{\extracolsep{20pt} } c@{\extracolsep{20pt}} c@{\extracolsep{20pt}} c@{\extracolsep{20pt}}c }
    \toprule
    Design &  Liability $K$ & Objective & Entry Fee $E^*$ & Reserve $\hat{e}^*$ & Shared Prize $S^*$  \\
    \midrule 
    High\_TW & $40$ & Both & $40$ & $136.5$ & $80$  \\ 
    Medium\_TW & $10.5$ & Both & $10.5$ & $30.5$ & $50.5$ \\ 
    Low\_W & $2$ & Winner & $2$ & $35.5$ & $10$ \\ 
    Low\_T & $2$ & Total & $2$ & $4.5$ & $42$\\ 
    Zero\_W & $0$ & Winner & $0$ & $36.5$ & $0$\\ 
    Zero\_T & $0$ & Total & $0$ & $0$ & N.A.\\ 
    \bottomrule
\end{tabular}
}
{\emph{Note.} High, Medium, Low, and Zero indicate the level of $K$, and $T$ and $W$ indicate whether the optimization objective is total effort or winner's effort. \cite{HuangZhangZhang26} prove that when $K$ is large enough, the optimal designs for the two objectives are identical, which is case for ``High\_TW'' and ``Medium\_TW''. ``Zero\_T'' is equivalent to a standard all-pay auction with no entry fee and no reserve; its optimal prize allocation is winner-takes-all, and prize sharing never occurs. 
}
\end{table}

Before each contest, participants were informed of the design and their privately known ability, which was drawn from the interval $[1, 2]$ uniformly at random. They then decided simultaneously whether to enter. 
A contestant who chose not to enter made no further decisions. Otherwise, he paid the entry fee and chose an effort level. Because entry decisions were simultaneous, contestants did not know how many others would enter, so they did not know the exact prize amount until the session ended. 

Across all designs, contestants' behavior differed markedly from game-theoretic predictions; see Figure~\ref{fig:contest-repl-ind}.
The theory predicts that (i) all contestants should enter regardless of ability, and (ii) low-ability contestants should exert zero effort, with a discontinuity in effort at the ability cutoff $t^*(K)$. 
In contrast, entry rates in the lab were well below 100\%, rising gradually from nearly zero for low-ability contestants to nearly 100\% for high-ability contestants. 
Moreover, low-ability contestants exerted positive effort, and effort increases smoothly with ability, showing no visible discontinuity.

After the experiments, a survey was administered to collect personal information from each participant, including demographics, risk attitudes, personality traits, and cognitive reflection test (CRT) score. The CRT score proxies impulsive versus reflective decision-making and is widely used in experimental economics to study strategic behavior 
\citep{frederick05CRT}. \cite{HuangZhangZhang26} used this information to explain the discrepancies between observed behavior and theoretical predictions. 
For example, while the theory assumes risk neutrality and full rationality, participants were generally risk averse and exhibited loss aversion, which is a common form of bounded rationality.

\subsection{LLM-MAS as a Design Evaluator}

We replicate the configurations of \cite{HuangZhangZhang26} to build LLM-MAS, with each LLM-powered agent simulating a contestant. 
The agents are powered by gpt-oss-20B, an open-source LLM from OpenAI.
The LLM prompts are adapted from the instructions given to lab participants, and each agent generates two actions: whether to enter and how much effort to exert. 
To account for individual heterogeneity, each agent is endowed with a \emph{persona} (i.e., attribute $\chi_i$) randomly drawn from the collected personal information. 
It is four-dimensional: gender, risk tolerance ($1 = $ ``not risk-taking at all'' to $7 = $ ``extremely risk-taking''), competitiveness ($1 = $ ``not competitive at all'' to $7 =$ ``extremely competitive''), and CRT score (0 to 3).

The simulation results are shown in Figure~\ref{fig:contest-repl-ind}. 
Across the six contest designs, 
LLM-powered agents' behavior closely matches that of human participants in the lab. 
As with humans, 
their effort increases with the endowed ability, 
and entry rates increase from almost zero for low-ability agents to nearly full participation for high-ability ones.
This alignment is remarkable given that we did not fine-tune or post-train the model, and it was released months before the publication of \cite{HuangZhangZhang26}, so data leakage is not a concern.

\begin{figure}[ht]
    \FIGURE{
    $
    \begin{array}{c}
    \includegraphics[width=\textwidth]{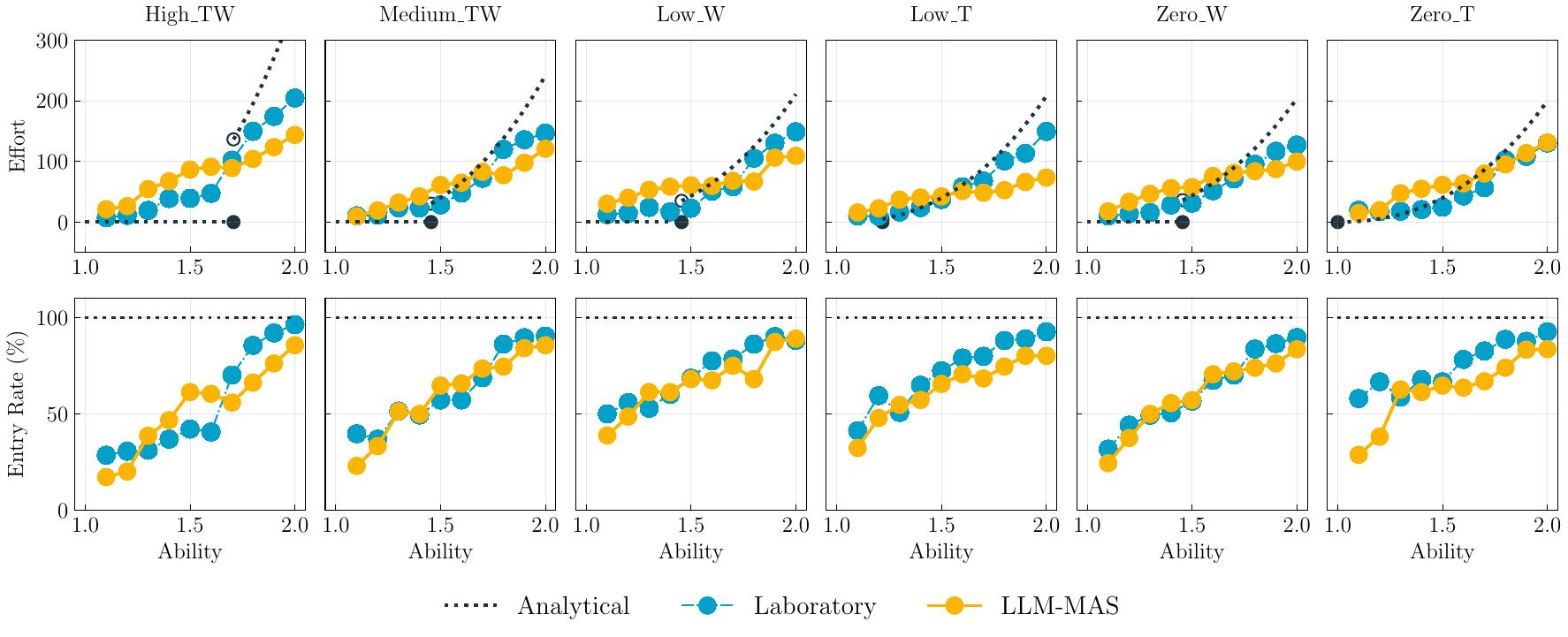} \label{fig:contest-repl-ind}
    \end{array}
    $
    }
    {Contestant Behavior Under Different Designs 
    }
    {Ability cutoff $t^*$ (left to right): $1.707$, $1.455$, $1.455$, $1.218$, $1.455$ and $0.000$.
    } 
\end{figure}

Because lab experiments are costly and time-consuming (the studies in \cite{HuangZhangZhang26} took months, whereas an equivalent number of experiments can be completed in days with LLM-MAS), these results highlight the great potential of LLM-MAS as a cost-effective evaluator for system designs.

In addition to its cost efficiency, LLM-MAS has another advantage over lab or field experiments: it allows us to capture agents' ``real-time'' reasoning processes, providing valuable insights into performance analysis.
In contrast, labs and field studies rarely observe such data except through post-experiment interviews.
In those cases, one can only infer the causes of deviations from theoretical predictions indirectly, based on behavioral outcomes.
The following examples show reasoning traces help explain why contestants' entry rate and effort mismatch the theory.

\begin{enumerate}[label=(\roman*)]
    \item Contrary to the predicted 100\% entry rate, an LLM-powered agent (male; ability 1.613; risk tolerance 2; competitiveness 2; CRT score 3) chose not to enter. His stated reasoning ``\texttt{Low risk tolerance leads to a conservative strategy of not entering to avoid potential loss from competition.}'' suggests that low risk tolerance drives his conservative action. His inclination towards avoiding losses over chasing uncertain gains reflects his self-reported risk attitude and competitive disposition.
    \item Contrary to the predicted zero effort from low-ability contestants, another LLM-powered agent (female; ability 1.287, below the cutoff of 1.802; risk tolerance 3; competitiveness 6; CRT score 2)  chose to enter and invested 100 points out of her 300-point endowment. She anticipated one other would enter as well. Her stated reasoning ``\texttt{I am moderately risk‑averse but highly competitive, so I choose to enter with a modest high investment just above the reserve to maximize potential payoff while keeping cost reasonable.}'' implies that high competitiveness counteracts risk aversion and leads to participation with a careful but proactive investment. 
    This also illustrates a trade-off between risk exposure and potential returns.
\end{enumerate}

\subsection{LLM-MAS as a Design Explorer}  

In the game-theoretic optimal design \eqref{eq:game-theo-optimal}, the parameters are coupled through the liability $K$: once the entry fee is fixed, the reserve and shared prize are consequently determined. 
Using LLM-MAS, we can relax this coupling and optimize $\theta = (E, \hat{e}, S)$ as independent parameters over the feasible set $\Theta = [0,300] \times [0, 1000] \times [0, 300]$. 
Consistent with the lab experiments, we set the upper bound of the entry fee to $K=300$.

The random quantity $\xi$ in the objective function corresponds to the vector of contestants' efforts, over which the expected total effort is computed. 
The contest is single-round, so $\xi$ does not evolve over time.
We can simulate this quantity in one round rather than running a long Markov chain to approximate the steady state, as in the supply chain example.
We therefore apply the MTL algorithm to optimize $\theta$. Here, the objective does not depend on $\theta$ explicitly, but only through the distribution of $\xi$. 
The GP method is thus not applicable, and we instead use the RF method for variance reduction. See Figure~\ref{fig:contest-opt-sol} for the iterates of the MTL-RF algorithm.

\begin{figure}[ht]
    \FIGURE{
    $
    \begin{array}{c}
    \includegraphics[width=\textwidth]{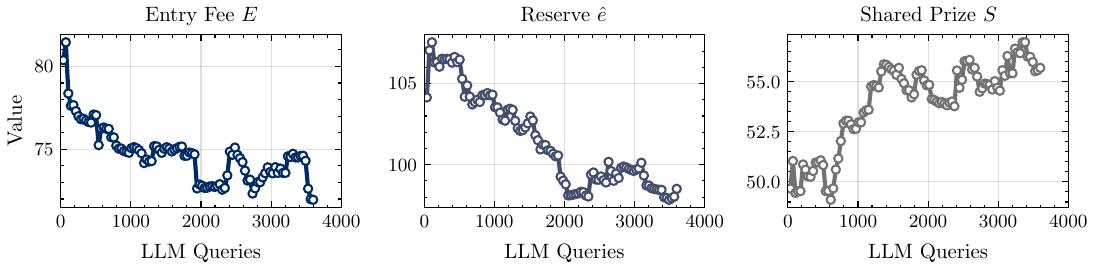} \label{fig:contest-opt-sol}
    \end{array}
    $
    }
    {Optimizing Contest Design Parameters
    }
    {}
\end{figure}

The resulting parameters are $E_{\text{OPT}} = 71.95$, $\hat{e}_{\text{OPT}} = 98.52$, and $S_{\text{OPT}} = 55.68$. 
The main difference from the game-theoretic design lies in the substantially lower entry fee.
According to \eqref{eq:game-theo-optimal}, the designer should set the entry fee as high as possible to maximize incentives by topping up the prize.
With $K=300$, this corresponds to an entry fee of $300$ and an exceptionally high reserve of $1163.60$. The expected total effort is then $232.94$, as given by \eqref{eq:game-theo-total-effort}. 

Under this high-fee, high-reserve design, the analytical model predicts that most contestants would exert zero effort, while 
only a few with abilities close to $b=2$ (because the cutoff $t^*$ is nearly $2$) exert substantial effort above the reserve to have a chance of winning.
However, LLM-MAS shows that with such high fees and reserves, even high-ability contestants opt out. They exert zero effort, like the others, to secure a high chance of receiving the shared prize, which is safer and easier than expending extremely high effort to compete for the top prize. 
This, in turn, drives total effort near zero, which is exactly the opposite of what the designer intended. Remarkably, setting only moderately high entry fees such as $E=90$ or $100$ is sufficient to reproduce this phenomenon via LLM-MAS. 

In contrast to the near-zero effort observed under high-entry barriers, LLM-MAS estimates that the design $(E_{\text{OPT}},\hat{e}_{\text{OPT}},S_{\text{OPT}})$ identified by our algorithm yields a total effort of 315.83, which even exceeds the analytical model's maximum prediction under unlimited liability (Figure~\ref{fig:contest-theo}).
See Figures~\ref{fig:contest-opt-sys} and \ref{fig:contest-opt-ind} for comparisons across designs in total effort and contestant behavior, respectively.

\begin{figure}[ht]
    \FIGURE{
    $
    \begin{array}{c}
    \includegraphics[width=0.75\textwidth]{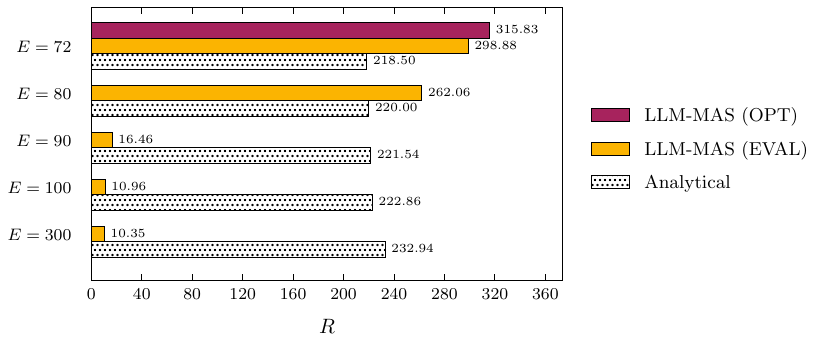} \label{fig:contest-opt-sys}
    \end{array}
    $
    }
    {Total Effort Under Contest Design Optimized via LLM-MAS     }
    {\begin{enumerate*}[label=(\roman*)]        
        \item LLM-MAS (OPT): total effort from LLM-MAS under the design identified by our algorithm.
        \item LLM-MAS (EVAL): total effort from LLM-MAS under the game-theoretic design \eqref{eq:game-theo-optimal}. 
        \item Analytical: total effort from the analytical model \eqref{eq:game-theo-total-effort} under the game-theoretic design \eqref{eq:game-theo-optimal}. 
    \end{enumerate*}}
\end{figure}

Furthermore, with the same entry fee of $E_{\text{OPT}} = 71.95$, our optimal design features a considerably lower reserve and shared prize ($\hat{e}_{\text{OPT}} = 98.52$, $S_{\text{OPT}} = 55.68$) than the game-theoretic design ($\hat{e}_{\text{ANLY}} = 258.06$, $S_{\text{ANLY}} = 111.64$).
As shown in Figure~\ref{fig:contest-opt-ind}, this configuration induces a much higher entry rate across ability levels, because it reduces downside risk and leads entry perceived as safer to risk-averse contestants.
Although it does not change the per-entrant effort, the larger pool of entrants yields higher total effort ($315.83$ versus $298.88$ in Figure~\ref{fig:contest-opt-sys}).

\begin{figure}[ht]
    \FIGURE{
    $
    \begin{array}{c}
    \includegraphics[width=0.9\textwidth]{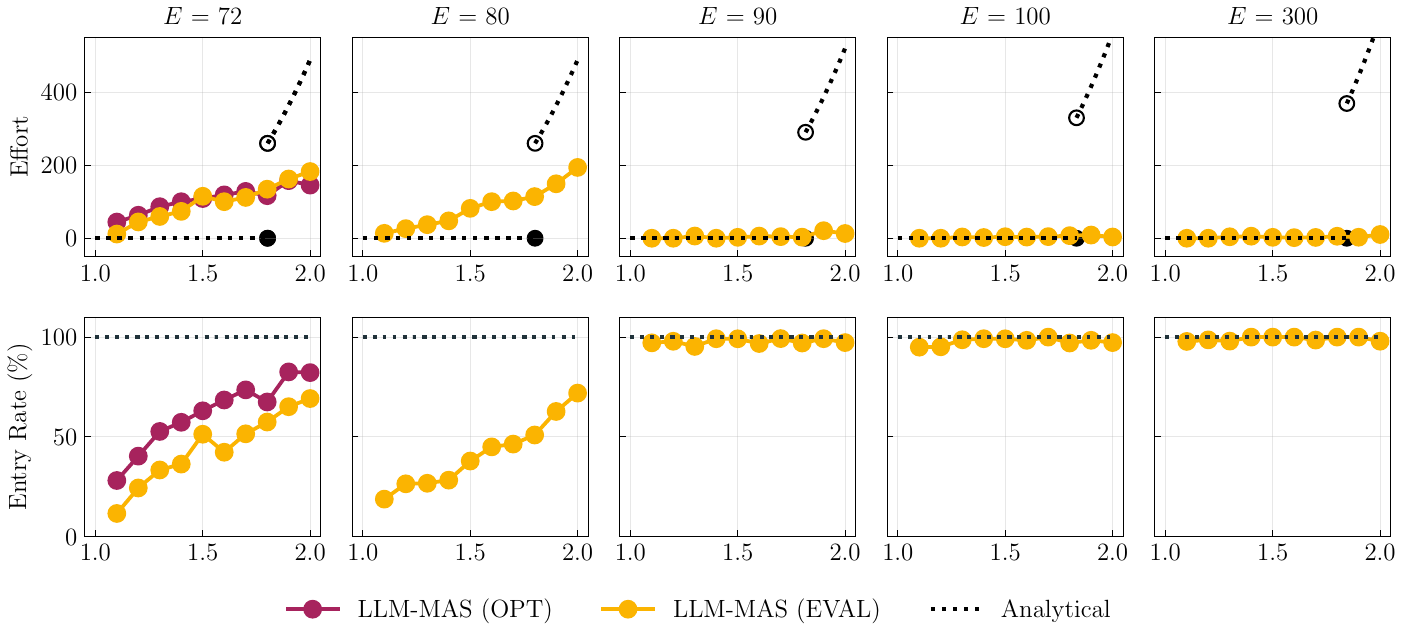} \label{fig:contest-opt-ind}
    \end{array}
    $
    }
    {Contestant Behavior Under Contest Design Optimized via LLM-MAS 
    }
    {\begin{enumerate*}[label=(\roman*)]        
        \item Ability cutoff $t^*$ (left to right): $1.802$, $1.816$, $1.832$, $1.845$ and $1.939$.
        \item LLM-MAS (OPT): behavior from LLM-MAS under the design identified by our algorithm.
        \item LLM-MAS (EVAL): behavior from LLM-MAS under the game-theoretic design \eqref{eq:game-theo-optimal}. 
        \item Analytical: behavior predicted by theory under the game-theoretic design \eqref{eq:game-theo-optimal}. 
    \end{enumerate*}
    } 
\end{figure}

By simulating how people actually behave, this case study reveals the failure of charging entry fees up to the liability limit and shows that lower reserves and smaller shared prizes can boost participation without increasing individual effort, which lifts total effort for the designer. 
This highlights LLM-MAS as a design explorer: it allows researchers and practitioners to discover and test alternative designs that standard analytical models alone would miss.

\section{Concluding Remarks} \label{sec:concl}
In this paper, we propose a principled framework for applying LLM-MAS to optimize service operations.
Compared with traditional numerical simulation, LLM‑MAS can generate and process rich information, including numerical, textual, and other unstructured data. 
This enables faithful representations of interactions, captures agent heterogeneity, and simulates boundedly rational behavior in response to both system designs and other agents' actions. 
We model the uncertainty within LLM‑MAS as a controlled Markov chain,
providing a theoretical basis for our OTL algorithm.
This algorithm simultaneously updates zeroth-order gradient estimates and optimization iterates on a single simulation trajectory, achieving significant savings in LLM queries. 
We further adapt two variance reduction techniques to handle the high variability from long horizon and multi-agent interactions. 

We demonstrate the practical value of our LLM‑MAS framework and algorithm through two applications.
In a supply-chain application, OTL and its variant solve the optimization problem effectively, while blackbox method and LLMs used as numerical solvers or as role‑playing system designers fail to produce meaningful results.
This also suggests that LLMs are a complement, not a substitute, for OR/OM methods. 
A contest-design case study with real behavioral data shows that LLM-MAS can serve two roles: a cost‑effective evaluator for testing known designs and an exploratory tool to uncover better designs that traditional models may overlook. It thus complements both analytical models and empirical studies via lab or field experiments.

Several directions remain to enhance the practical value of LLM-MAS. 
For example, scaling LLM‑MAS to large systems with many agents will require tighter control of query costs without losing significant fidelity. 
Promising approaches include surrogate modeling at both the agent and system levels, streamlined multi‑agent communication, and stronger variance-reduction methods for long horizons and dense interactions. 
In addition, new algorithms are needed for high‑dimensional design spaces, since zeroth‑order methods degrade as dimension grows. 
Another direction is to extend LLM‑MAS with a broader ecosystem of tools and multimodal capabilities, so it can simulate more complex human activities and handle images, audio, video, and graphs beyond text. This would enable richer simulation environments and support studies such as simulating AI-assisted sales teams to develop better incentive mechanisms, and simulating consumer reactions to different ad creatives to improve platform operations.

\begin{APPENDIX}{Assumptions for Asymptotic Convergence} \label{app:assmptionption}

\begin{assumption}[Feasible Set] \label{ass:feasible-set} $\Theta$ is a compact and convex subset of $\RR^d$.     
\end{assumption}

\begin{assumption}[Objective Function] \label{ass:obj-f}
$f(\theta)$ is continuously differentiable, and there exists a constant $L_f>0$ such that $\|\nabla f(\theta) - \nabla f(\theta')\| \le L_f \|\theta-\theta'\|$ for all $\theta,\theta' \in \Theta$.
\end{assumption}

\begin{assumption}[Random Function] \label{ass:rand-F}
    There exist constants $L_F > 0$ and $C_F > 0$ such that $|F(\theta;\xi) - F(\theta';\xi)| \le L_F \|\theta-\theta'\|$ for all $\theta,\theta' \in \Theta$ and all $\xi \in \Xi$, and 
    $|F(\theta;\xi)| \le C_F$ for all $\theta \in \Theta$ and all $\xi \in \Xi$.
\end{assumption}

\begin{assumption}[Transition Kernel] \label{ass:trans-kern}
For each $\theta\in\Theta$,
let $\msfP_\theta$ denote the Markov chain's transition kernel. 
    There exists a constant $L_{\msfP}>0$ such that
    $\scrW_1(\msfP_\theta(\xi, \cdot), \msfP_{\theta'}(\xi, \cdot)) \leq L_{\msfP} \| \theta-\theta' \|$ for all $\theta, \theta' \in \Theta$ and all $\xi \in \Xi$, where $\scrW_1(\nu,\nu') = \inf_{\pi \sim \mathsf{\Gamma}(\nu,\nu')}\E_{\zeta,\zeta' \sim \pi}[\|\zeta-\zeta'\|_1]$ is 1-Wasserstein distance with $\mathsf{\Gamma}(\nu,\nu')$ being the set of all couplings of probability measures $\nu$ and $\nu'$. 
\end{assumption}

\begin{assumption}[Geometric Mixing] \label{ass:uni-erg} 
    For each $\theta \in \Theta$,
    $\msfP_\theta$ admits a unique stationary distribution $\mu_\theta$.
    Moreover, there exist constants $C_M > 0$ and $\gamma \in (0,1)$ such that $ \|\msfP_\theta^{(k)}(\xi,\cdot) - \mu_\theta(\cdot)\|_{\TV} \le C_M \gamma^k$ for all $ k \ge 0$ and all $\xi \in \Xi$, where $\msfP_\theta^{(k)}(\xi, \cdot) = \pr(\xi^{(k)} \in \cdot \vert \xi^{(0)} = \xi; \theta)$ is the $k$-step transition distribution from the initial state $\xi$, and $\|\nu-\nu'\|_{\TV} = 1/2 \int |\nu(\zeta) - \nu'(\zeta)| \, \dd  \zeta$ is the total variation distance between probability measures $\nu$ and $\nu'$.
\end{assumption}

\begin{assumption}[Transient Gradient]\label{ass:abs-conti}
    For each $\theta\in\Theta$ and all $\xi\in\Xi$, 
    there exists a function $H(\theta;\xi)$ such that 
    $\nabla_\theta \int_{\Xi} F(\theta;\xi') \msfP_\theta(\xi, \dd \xi') = \int_{\Xi} H(\theta;\xi') \msfP_{\theta}(\xi ,\dd \xi')$.    
\end{assumption}

Under Assumption~\ref{ass:abs-conti}, we have $\nabla  f(\theta) = \E_{\xi\sim \mu_\theta}[H(\theta; \xi)]$ by the invariance of $\mu_\theta$.
For any given $\theta$, this yields the Poisson equation: \begin{align}
     H(\theta;\xi) - \nabla f(\theta) 
    = \nu_\theta(\xi) - \int_{\Xi} \nu_\theta(\xi') \msfP_\theta (\xi,\dd{\xi}').\label{eq:peq}
\end{align}
Assumption~\ref{ass:uni-erg} ensures the existence of a solution 
$\nu_\theta(\xi)$; see \citet[Chapter~17]{MeynTweedie12}.

\begin{assumption}[Poisson's Equation]  \label{ass:peq-sol}
There exist constants $L_{\nu} > 0$ and $C_H > 0$ such that 
\begin{enumerate}[label=(\roman*)]
    \item %
    $\sup_{\theta \in \Theta} \|\nu_\theta(\xi) - \nu_\theta(\xi')\| \le L_{\nu} \|\xi - \xi'\|$ for all $\xi,\xi' \in \Xi$; \label{ass:peq-sol-xi}
    \item %
    $\sup_{\theta \in \Theta} \|\nu_\theta(\xi)\| \le L_{\nu} (1 + \|\xi\|)$  for all  $\xi \in \Xi$;  \label{ass:peq-sol-lingrowth}
    \item %
    $\|\nu_\theta(\xi) - \nu_{\theta'}(\xi)\| \le L_{\nu} \|\theta - \theta'\|(1 + \|\xi\|)$  for all  $\theta,\theta' \in \Theta$ and all $\xi \in \Xi$; and  \label{ass:peq-sol-theta} 
    \item $\E_{\xi' \sim \msfP_{\theta}(\xi,\cdot)}[\|H(\theta;\xi')\|^2] \le C_H$ for all $\theta \in \Theta$ and all $\xi \in \Xi$.  \label{ass:mom-H}
\end{enumerate}
\end{assumption}

\begin{assumption}[Stepsizes] \label{ass:stepsize}
   $\eta_k \to 0$ and $\delta_k \to 0$ as $k\to\infty$; $\sum_{k=0}^\infty \eta_k = \infty$, $\sum_{k=0}^\infty \eta_k \delta_k < \infty$, and $\sum_{k=0}^\infty \eta_k^2/\delta_k^2 < \infty$.
\end{assumption}

\end{APPENDIX}

\bibliographystyle{informs2014} %
\bibliography{ref} %

\ECSwitch

\EquationsNumberedBySection

\ECHead{Supplemental Material}

\section{Details of the Supply Chain Example} \label{ec:ex}

\subsection{System State}

Similar to the retailer ($\msfR$) discussed in Section \ref{sec:simsys-ex}, the elements of system state $\xi^{(t)}$ relevant to the manufacturer ($\msfM$) and the consumer ($\msfC$) are detailed as follows.

\paragraph{Manufacturer.} 
The manufacturer's action $A_{\msfM}^{(t)}$ consists of a wholesale price $\WS^{(t)}$ and an investment level in low-carbon technology $\TECH^{(t)}$.
These are chosen based only on information from the previous round, $H_{\msfM}^{(t-1)}$. 
After acting, the manufacturer observes the carbon emissions $\EMS^{(t)}$ and communicates the wholesale price $\WS^{(t)}$ and the carbon footprint $\FP^{(t)}$ (the percentage reduction in emissions relative to a baseline) to the retailer.   
Later, after the other two agents have acted, the manufacturer receives the purchasing quantity $\QUT^{(t)}$, which is part of the consumer's action, via the retailer. 

The manufacturer has no pre-action observation or pre-action communications; i.e., $O_{\msfM}^{(t,\preA)} = M^{(t,\preA)} = \emptyset$.
Its post-action observation is 
$O_{\msfM}^{(t, \postA)} = \EMS^{(t)}$. 
Its post-action communications include the message sent to the retailer $M_{\msfM \to \msfR}^{(t,\postA)} = (\WS^{(t)}, \FP^{(t)})$ and the message later received from the retailer $M_{\msfM \gets \msfR}^{(t,\postA)} = \QUT^{(t)}$.
Thus, at the end of this round, the manufacturer's memory module stores  
\begin{align*}
    H_{\msfM}^{(t)} = (\overunderbraces{&&\br{2}{M^{(t, \postA)}_{\msfM \to \msfR}}}%
{&\TECH^{(t)}, & \WS^{(t)}, & \FP^{(t)}, & \EMS^{(t)}, & \QUT^{(t)}}%
{&\br{2}{A_{\msfM}^{(t)}} & & \br{1}{O_{\msfM}^{(t,\postA)}} & \br{1}{M^{(t,\postA)}_{\msfM \gets \msfR}}}).
\end{align*}

\paragraph{Consumer.}  The consumer's action $A_{\msfC}^{(t)}$ consists of a purchasing quantity $\QUT^{(t)}$ and a willingness to pay $\WTP^{(t)}$. 
These are chosen using the information from the previous round, $H_{\msfC}^{(t-1)}$ and the pre‑action message received from the retailer $M^{(t, \preA)}_{\msfC \gets \msfR}$. After acting, the consumer communicates the purchasing quantity $\QUT^{(t)}$ to the retailer.

The consumer has no pre-action observation; i.e., $O_{\msfC}^{(t,\preA)} = \emptyset$.
Its pre-action communication is the message from the retailer $M^{(t, \preA)}_{\msfC \gets \msfR} = (\RT^{(t)},\AD^{(t)})$. Its post-action communication is the message sent to the retailer $M^{(t, \postA)}_{\msfC \to \msfR} = \QUT^{(t)}$. There is no post-action observation. Thus, at the end of this round, the consumer's memory module stores 
\begin{align*}
    H_{\msfC}^{(t)} = (\overunderbraces{&&&&\br{1}{M^{(t, \postA)}_{\msfC \to \msfR}}}%
{&\AD^{(t)}, & \RT^{(t)}, &\WTP^{(t)}, & \QUT^{(t)}}%
{& \br{2}{M^{(t, \preA)}_{\msfC \gets \msfR}} &\br{2}{A_{\msfC}^{(t)}}}).
\end{align*}

\subsection{Objective Function}

Following the notations in Example~\ref{ex:scm},
the policymaker's objective function is given by
\[F(\theta; \xi) = -\big(\SCWF(\theta;\xi) - \FISC(\theta;\xi) - \ENV(\xi)\big),\]
which captures trade-offs among supply chain welfare ($\SCWF$), fiscal impact ($\FISC$), and environmental externalities ($\ENV$).

\begin{enumerate}[label=(\roman*)]
    \item The supply chain welfare is $\SCWF(\theta;\xi) = \MPF + \RPF + \texttt{CS}$.
    Specifically, the manufacturer's profit is 
    \[\MPF = (\WS - c^{\text{prod}}) \QUT - 0.5 c^{\text{tech}} \TECH^2 - \theta_1 \EMS \cdot \QUT,\] 
    where $c^\text{prod}$ and $c^{\text{tech}}$ are the per-unit cost of production and technology adoption, 
    respectively. 
    The retailer's profit is 
    \[\RPF = (\RT - \WS) \QUT - \MKT.\] The consumer surplus is \[\texttt{CS} = (\WTP - \RT + \theta_{2}) \QUT.\] 
    \item The fiscal impact is measured via an asymmetric function, given by 
    \[\texttt{FISC}(\theta;\xi) = ((\texttt{EXP}-c^{\text{tag}})^+)^{c_+} + ((c^{\text{tag}} - \texttt{EXP})^+)^{c_-}.\]
    The policy expenditure is \[\texttt{EXP} = - \theta_1 \EMS \cdot \QUT + \theta_{2} \QUT,\]
    where the first term is the total emission taxes $(\theta_1)$ charged on the manufacturer, and 
    the second term is the subsidies $(\theta_2)$ granted to the consumer. 
    The parameter $c^{\text{tag}}$ denotes the target spending, and $c_+$ and $c_-$ are the penalty intensities for overspending $(\texttt{EXP}-c^{\text{tag}})^+$ and underspending $(c^{\text{tag}} - \texttt{EXP})^+$, respectively; usually, $c_+ > c_-$ meaning overspending is penalized more heavily.
    \item The environmental externality is computed as
    \[
    \texttt{ENV}(\xi) = c^{\text{env}}_1 (\EMS \cdot \QUT)^{c^{\text{env}}_2},\]
    where the parameter $c^{\text{env}}_1$ represents the monetary cost of emissions, and $c^{\text{env}}_2$ controls the curvature.
    The carbon emissions evolve according to $\EMS^{(t+1)} = ( \EMS^{(t)}- \Delta \EMS^{(t)} )^+$, which remains nonnegative. In particular, 
    \[\Delta \EMS^{(t)} =  e^{\text{red}} (\EMS^{(t)} (1 + \zeta) - e^{\text{base}}) \log (1 + \TECH^{(t)}).\]
    Here, $t$ corresponds to the round index in LLM-MAS.
    The parameter $e^{\text{base}}$ represents the minimal achievable emission level and $e^{\text{red}}$ is the efficiency of emission reduction per unit of technological effort. 
    The random variable $\zeta \sim N(0,\sigma^2)$ captures exogenous uncertainty in the emission process.
     The reduced emissions $\Delta \EMS^{(t)}$ depend on the extent to which the realized emissions exceed $e^{\text{base}}$, which is scaled by a logarithmic term of technological effort $\TECH^{(t)}$. This captures diminishing returns to technology: initial investments yield significant reductions, whereas subsequent investments provide progressively smaller gains.
\end{enumerate}

\subsection{Agent Attributes}

The attributes $\chi_i$ in the agent's profile module are configured to capture key behavioral characteristics. For the retailer, we include ``willingness to collaborate'', which is equally likely to be high, moderate, or low. This reflects the retailer's tendency to work with the manufacturer on promoting technology adoption through marketing efforts. For the consumer, we include ``sustainability awareness'', which is likewise equally likely to be eco-aware, eco-neutral, or eco-skeptical. This captures the degree to which the consumer is inclined to purchase sustainable products.

\subsection{Experimental Configurations} 

The parameters in $F(\theta;\xi)$ are specified as follows. 
The parameters related to $\SCWF(\theta;\xi)$ are generated as $c^{\text{prod}} = 1 + \Unif[0,1]$ and $c^{\text{tech}} \sim \Unif[0.5,1]$;
both are sampled once at initialization and then kept fixed throughout the experiment.
The parameters related to $\FISC(\theta;\xi)$ are set to $c_+ = 1.2, c_-=0.8$, and $c^{\text{tag}} = 0$.
The emission-related parameters are given by $\EMS^{(0)} = e^{\text{init}} = 8$, $e^{\text{red}} = 0.05$, $e^{\text{base}} = 3$, and $\sigma = 0.5$. Finally, the parameters related to $\ENV(\xi)$ are set to $c^{\text{env}}_1 = 0.05$ and $c^{\text{env}}_2 = 1.2$. 

The agents' action spaces are defined as follows. For the manufacturer, $\TECH \in [2,5]$ and $\WS \in [6,8]$; for the retailer, $\MKT \in [20,30]$ and $\RT \in [12,15]$; and for the consumer, $\WTP \in [15,18]$ and $\QUT \in [5,15]$. Advertising expenses are categorized into three levels, taking the values of 15, 20, and 25 for low-, medium-, and high-quality advertisement, respectively. 

Hyperparameters of the OTL algorithm are set to $\alpha = 0.75$, $\beta = 1$, 
and $\rho = 0.9$. 
BO is implemented using a Gaussian process with a Mat\'ern kernel (smoothness $2.5$, lengthscale $1$) as the surrogate model and expected improvement as the acquisition function. 
Hyperparameters of LLaMA3.1-8B are configured as follows: ``temperature'' is set to 1, ``top\_p'' is set to 0.9, and ``max\_new\_tokens'' is capped  at 1000.

\section{Proof of Asymptotic Convergence} \label{ec:proof}

\subsection{Technical Lemmas} \label{ec:tech-lem}

\begin{lemma} \label{lem:gp-bias}
    Suppose the positive definite matrix $\Sigma(\xi)$ is  given. Under Assumptions~\ref{ass:feasible-set}--\ref{ass:obj-f}, for all $\theta \in \Theta$, 
    \begin{align*}
        \E\left[ G^{\GP}(\theta; \xi, \xi^{\pm}, u, \delta) \right] 
        = \Sigma(\xi) \nabla f_{\delta, \msfLambda}(\theta), \quad u \sim \msfLambda = N(0,\Sigma(\xi)).
    \end{align*}
\end{lemma}

\proof{Proof of Lemma~\ref{lem:gp-bias}.}
By definition, the smoothed gradient can be expanded as follows:
\begin{align*}
    &\quad \nabla f_{\delta,\msfLambda}(\theta) 
    = \frac{1}{\sqrt{(2\pi)^d|\Sigma(\xi)|}} \int \exp\left(-\frac{1}{2} u^\intercal \Sigma(\xi)^{-1} u\right) \nabla_\theta f(\theta+ \delta u) \, \dd u = \E_{u \sim \msfLambda} [\nabla f(\theta+\delta u)] \\
    &= \frac{\delta^{-d}}{\sqrt{(2\pi)^d|\Sigma(\xi)|}} \int \exp\left(-\frac{1}{2\delta^2} (\vartheta-\theta)^\intercal \Sigma(\xi)^{-1} (\vartheta-\theta)\right) \nabla_\vartheta f(\vartheta) \, \dd \vartheta 
    \\
    &= \frac{\delta^{-d}}{\sqrt{(2\pi)^d|\Sigma(\xi)|}} \left. \exp \left(-\frac{1}{2\delta^2} (\vartheta-\theta)^\intercal \Sigma(\xi)^{-1} (\vartheta-\theta)\right) f(\vartheta) \right|_{-\infty}^{+\infty} \\
    &\quad - \frac{\delta^{-d}}{\sqrt{(2\pi)^d|\Sigma(\xi)|}} \int f(\vartheta) \nabla_\vartheta \exp \left(-\frac{1}{2\delta^2} (\vartheta-\theta)^\intercal \Sigma(\xi)^{-1} (\vartheta-\theta)\right) \, \dd \vartheta \\
    &= \frac{\delta^{-(d+2)}}{\sqrt{(2\pi)^d|\Sigma(\xi)|}} \int f(\vartheta) \exp \left( -\frac{1}{2\delta^2} (\vartheta-\theta)^\intercal \Sigma(\xi)^{-1} (\vartheta-\theta) \right) \Sigma(\xi)^{-1} (\vartheta-\theta) \, \dd \vartheta \\
    &= \frac{\delta^{-2}}{\sqrt{(2\pi)^d|\Sigma(\xi)|}} \int f(\theta+\delta u) \exp\left(-\frac{1}{2} u^\intercal \Sigma(\xi)^{-1} u\right) \Sigma(\xi)^{-1} (\delta u) \, \dd u \\
    &= \Sigma(\xi)^{-1} \int \frac{1}{\sqrt{(2\pi)^d|\Sigma(\xi)|}} \exp\left(-\frac{1}{2} u^\intercal \Sigma(\xi)^{-1} u\right) \frac{f(\theta+\delta u)}{\delta} u \, \dd u \\
    &= \Sigma(\xi)^{-1} \E_{u \sim \msfLambda}\left[\frac{f(\theta+\delta u)}{\delta}u\right] \\
    &= \Sigma(\xi)^{-1} \E_{u \sim \msfLambda}\left[\left. \E_{\xi^+ \sim \mu_{\theta+\delta u}}\left[\frac{F(\theta+\delta u;\xi^+)}{\delta}u \right\vert u \right] \right] \\
    &= \Sigma(\xi)^{-1} \E_{u \sim \msfLambda}\left[\left. \E_{\xi^\pm \sim \mu_{\theta\pm\delta u}}\left[\frac{F(\theta+\delta u; \xi^+) - F(\theta-\delta u; \xi^-)}{2\delta} u \right\vert u \right] \right] \\
    &=  \Sigma(\xi)^{-1} \E[G^{\GP}(\theta; \xi, \xi^{\pm}, u, \delta)],
\end{align*}
where the interchange of integration and differentiation  
follows from the $L$-smoothness of $f$ and the compactness of $\Theta$. Here, $\left. \exp\left(-\frac{1}{2\delta^2} (\vartheta-\theta)^\intercal \Sigma(\xi)^{-1} (\vartheta-\theta)\right) f(\vartheta) \right|_{- \infty}^{+\infty} = 0$ holds because 
the Gaussian tail dominates. The last equality is due to the symmetry of Gaussian.
The bias of the GP gradient estimator is thus given by $\left\|\E\left[ G^{\GP}(\theta; \xi, \xi^{\pm}, u, \delta) \right] - \nabla f_{\delta, \msfLambda}(\theta) \right\| = \|(\Sigma(\xi) - I_d) \nabla f_{\delta, \msfLambda}(\theta)\|$.
\Halmos \endproof

\begin{lemma} \label{lem:trans-peq-sol}
    Under Assumptions~\ref{ass:trans-kern}--\ref{ass:peq-sol}, there exists a constant $L_{\msfP\nu} > 0$ such that
\begin{enumerate}[label=(\roman*)]
    \item $\sup_{\theta \in \Theta} \|\msfP_\theta \nu_\theta(\xi)\| \le L_{\msfP\nu}(1 + \|\xi\|)$ for all $\xi \in \Xi$; and  \label{lem:trans-peq-sol-lingrowth}
    \item $\|\msfP_\theta \nu_\theta(\xi) - \msfP_{\theta'} \nu_{\theta'}(\xi)\| \le L_{\msfP\nu}\|\theta - \theta'\|(1 + \|\xi\|)$ for all $\theta,\theta' \in \Theta$ and all $\xi \in \Xi$.\label{lem:trans-peq-sol-theta}
\end{enumerate}
Here, we write $\msfP_\theta f (\xi) = \int \msfP_\theta(\xi, \dd \xi') f(\xi')$ for any measurable function $f$.
\end{lemma}

\proof{Proof of Lemma~\ref{lem:trans-peq-sol}.}

Assumption~\ref{ass:uni-erg} implies that the following drift condition holds: there exists a Lyapunov function $V: \Xi \mapsto [1, +\infty)$, independent of $\theta$, such that $\msfP_\theta V(\xi) \leq C_{V} V(\xi)+C_{V}'$ for some positive constants $C_{V} < 1$ and $C_{V}'  < \infty$.  
See \citet[Chapter~16]{MeynTweedie12}.

For (i), $\|\msfP_\theta \nu_\theta(\xi)\| \le L_{\nu} \msfP_\theta (1+\|\xi\|) \le L_{\nu} (1 + C_V \|\xi\| + C_V') \le L_{\msfP\nu}(1 + \|\xi\|)$,
 with $L_{\msfP\nu} \coloneqq L_{\nu} \max\{1+C_V, C_V'\}$. The first inequality is 
due to 
Assumption~\ref{ass:peq-sol}-\ref{ass:peq-sol-lingrowth}, and the second inequality is by the drift condition.

For (ii), by triangle inequality, $\|\msfP_\theta \nu_\theta(\xi) - \msfP_{\theta'} \nu_{\theta'}(\xi)\| \le \|\msfP_\theta \nu_\theta(\xi) - \msfP_{\theta'} \nu_\theta(\xi) \| + \|\msfP_{\theta'} \nu_\theta(\xi) - \msfP_{\theta'} \nu_{\theta'}(\xi) \|$. For the first part, we have 
\begin{align}
    &\quad \|\msfP_\theta \nu_\theta(\xi) - \msfP_{\theta'} \nu_\theta(\xi) \| = \left\| \int \nu_\theta(\xi') \left( \msfP_\theta \nu_\theta - \msfP_{\theta'} \nu_\theta \right) (\xi, \dd \xi') \right\| \nonumber \\
    &\le \left( \sup_{\xi\neq\xi'} \frac{\|\nu_\theta(\xi)-\nu_\theta(\xi')\|}{\|\xi-\xi'\|} \right) \cdot \scrW_1(\msfP_\theta(\xi,\cdot), \msfP_{\theta'}(\xi,\cdot)) = \Lip(\nu_{\theta}) \cdot \scrW_1(\msfP_\theta(\xi,\cdot), \msfP_{\theta'}(\xi,\cdot)) \nonumber \\
    &\le L_{\nu} L_{\msfP}  \|\theta-\theta'\| \label{eq:lem-ec2},
\end{align}
where the first inequality applies the Kantorovich-Rubinstein duality and the second inequality is by Assumption~\ref{ass:trans-kern} and Assumption~\ref{ass:peq-sol}-\ref{ass:peq-sol-theta}. For the second part, 
\begin{align*}
    &\quad \|\msfP_{\theta'} \nu_\theta(\xi) - \msfP_{\theta'} \nu_{\theta'}(\xi) \| = \left\| \int \left( \nu_\theta(\xi') - \nu_{\theta'}(\xi') \right) \msfP_{\theta'} (\xi, \dd \xi') \right\| \\
    &\le \msfP_{\theta'} \left( L_{\nu} \|\theta-\theta'\|(1+\|\xi\|) \right) \le L_{\nu} \|\theta-\theta'\|(1+C_V \|\xi\| + C_V') \le L_{\msfP\nu} \|\theta-\theta'\| (1 + \|\xi\|),
\end{align*}
where the first inequality is by Assumption~\ref{ass:peq-sol}-\ref{ass:peq-sol-theta} and the second inequality follows from (i).
Combining the two parts 
completes the proof, where $L_{\msfP\nu} \coloneqq L_{\nu} L_{\msfP}$.
\Halmos \endproof

For the subsequent analysis,  we define $\hat{f}(\theta;\xi) = \int_{\Xi} F(\theta;\xi') \msfP_\theta(\xi, \dd \xi')$ and $\hat{f}_{\delta,\msfLambda}(\theta;\xi) = \int_{\RR^d} \hat{f}(\theta+\delta u;\xi) \, \msfLambda(\dd u)$.
Here, $\hat{f}(\theta;\xi)$ is the expected value of $F(\theta;\xi') $ with $\xi'$ drawn from the transition distribution $\msfP_\theta(\xi, \cdot)$ given the current state $\xi$, and $\hat{f}_{\delta,\msfLambda}(\theta;\xi)$ denotes its smoothed counterpart with $\msfLambda = N(0,I_d)$. This construction parallels $f(\theta)$ and $f_{\delta}(\theta)$ in \eqref{eq:so} and \eqref{eq:smoothed-f}, except that the expectation is now taken over $\msfP_\theta$ in lieu of $\mu_\theta$.

\begin{lemma} \label{lem:approx-bd}
    Suppose Assumptions~\ref{ass:rand-F}--\ref{ass:peq-sol} hold, and let $\msfLambda = N(0,I_d)$. Then for any $\xi \in \Xi$, 
    \begin{align*}
        \left\|\E_{u \sim \msfLambda, \xi^\pm \sim \msfP_{\theta \pm \delta u(\xi,\cdot)}} [G(\theta;\xi^\pm, u,\delta)] - \E_{\xi' \sim \msfP_\theta(\xi ,\cdot)}[H(\theta;\xi')] \right\| \le \delta L_{\hat{f}} (d+3)^{3/2}/2,
    \end{align*}
    where $L_{\hat{f}} > 0$ is some constant.
\end{lemma}

\proof{Proof of Lemma~\ref{lem:approx-bd}.}
We first show that $H(\theta;\xi)$ is Lipschitz continuous in $\theta$ for any given $\xi \in \Xi$. Recall the Poisson equation $H(\theta;\xi) - \nabla f(\theta) 
= \nu_\theta(\xi) - \E_{\xi'\sim\msfP_\theta(\xi,\cdot)}[\nu_\theta(\xi')]$. By triangle inequality, we have
\begin{align*}
    \left\| H(\theta;\xi) - H(\theta';\xi) \right\| &= \left\| \left( \nu_\theta(\xi) - \E_{\xi'\sim\msfP_\theta(\xi,\cdot)}[\nu_\theta(\xi')] \right) - \left( \nu_{\theta'}(\xi) - \E_{\xi'\sim\msfP_{\theta'}(\xi,\cdot)}[\nu_{\theta'}(\xi')] \right) \right\| \\
    &\le \underbrace{\left\|\nu_\theta(\xi) - \nu_{\theta'}(\xi)\right\|}_{\coloneqq \Tone} + 
    \underbrace{\left\|\E_{\xi' \sim \msfP_\theta(\xi,\cdot)}[\nu_\theta(\xi')] - \E_{\xi'\sim \msfP_{\theta'}(\xi,\cdot)}[\nu_{\theta'}(\xi')]\right\|}_{\coloneqq \Ttwo}.
\end{align*}

Bound term $\Tone$. It is easy to show that $\Tone \le L_{\nu}(1 + C_\xi)\|\theta - \theta'\|$, where $C_\xi = \|\xi\|$ by Assumption~\ref{ass:peq-sol}-\ref{ass:peq-sol-theta}. 

Bound term $\Ttwo$. Observe that 
\begin{align*}
    \Ttwo = \left\|\E_{\msfP_\theta}[\nu_\theta] - \E_{\msfP_{\theta'}}[\nu_{\theta'}]\right\| \le \underbrace{\left\| \E_{\msfP_\theta}[\nu_\theta -\nu_{\theta'}] \right\|}_{\coloneqq \Tthree} + \underbrace{\left\| \E_{\msfP_\theta}[\nu_{\theta'}] - \E_{\msfP_{\theta'}}[\nu_{\theta'}] \right\|}_{\coloneqq \Tfour},
\end{align*}
where $\msfP_\theta$ and $\msfP_{\theta'}$ are shorthands for $\xi'\sim\msfP_{\theta}(\xi,\cdot)$ and $\xi'\sim\msfP_{\theta'}(\xi,\cdot)$, respectively.
By Assumption~\ref{ass:peq-sol}-\ref{ass:peq-sol-theta}, we have $\Tthree \le L_{\nu}(1 + C_\xi)\|\theta - \theta'\|$. Applying the Kantorovich-Rubinstein duality, we obtain $\Tfour \le \Lip(\nu_{\theta'}) \cdot \scrW_1(\msfP_\theta(\xi,\cdot), \msfP_{\theta'}(\xi,\cdot)) \le L_{\nu} L_{\msfP}\|\theta-\theta'\|$
by Assumption~\ref{ass:peq-sol}-\ref{ass:peq-sol-xi}. Combining the terms $\Tone$, $\Tthree$ and $\Tfour$ gives us $\left\| H(\theta;\xi) - H(\theta';\xi) \right\| \le L_H \|\theta-\theta'\|$, where $L_H = 2L_{\nu}(1 + C_\xi) + L_{\nu} L_{\msfP}$. 
It then follows from Assumption~\ref{ass:abs-conti} that
\begin{align}
    &\quad \left\| \nabla_\theta \hat{f}(\theta;\xi) - \nabla_{\theta'} \hat{f}(\theta';\xi) \right\| = \left\| \nabla_\theta \left( \int_{\Xi} F(\theta;\xi') \msfP_\theta(\xi, \dd \xi') \right) - \nabla_{\theta'} \left( \int_{\Xi} F(\theta';\xi') \msfP_{\theta'}(\xi, \dd \xi') \right) \right\| \nonumber \\
    &= \left\| \E_{\xi' \sim \msfP_{\theta}(\xi ,\cdot)}[H(\theta;\xi')] - \E_{\xi' \sim \msfP_{\theta'}(\xi ,\cdot)}[H(\theta';\xi')] \right\| \nonumber \\
    &\le \left\| \int \left( H(\theta;\xi') - H(\theta';\xi') \right) \, \msfP_{\theta}(\xi, \dd \xi') \right\| + \left\| \int H(\theta';\xi')  (\msfP_{\theta} - \msfP_{\theta'}) (\xi, \dd \xi') \right\| \nonumber \\
    &\le L_H \|\theta-\theta'\| + L_H \scrW_1 \left(\msfP_{\theta}(\xi,\cdot), \msfP_{\theta'}(\xi,\cdot) \right) \nonumber \\
    &\le L_{\hat{f}} \|\theta-\theta'\|, \label{eq:lem-ec3-Lf}
\end{align}
where the penultimate inequality follows once again from the Kantorovich-Rubinstein duality, and in the last inequality we take $L_{\hat{f}} = L_H (1+L_{\msfP})$. Note that
\begin{align}
    &\quad \E_{u \sim \msfLambda, \xi^\pm \sim \msfP_{\theta \pm \delta u(\xi,\cdot)}} [G(\theta;\xi^\pm, u,\delta)] \nonumber \\
    &= \int_{\RR^d} \int_{\Xi \times \Xi} \frac{F(\theta + \delta u; \xi^+) - F(\theta - \delta u; \xi^-)}{2\delta} u \, \msfP_{\theta + \delta u}(\xi, \dd{\xi}^{+}) \, \msfP_{\theta - \delta u}(\xi, \dd{\xi}^{-}) \, \msfLambda (\dd u) \nonumber \\
    &= \int_{\RR^d} \int_{\Xi} \frac{F(\theta + \delta u; \xi^+)}{\delta} u \, \msfP_{\theta + \delta u}(\xi, \dd{\xi}^+) \, \msfLambda (\dd u) = \frac{1}{\delta} \int_{\RR^d} \hat{f}(\theta + \delta u; \xi) u \, \msfLambda (\dd u) \label{eq:lem-ec3-fhat} \\
    &= -\frac{1}{\delta} \int_{\RR^d} \hat{f}(\theta + \delta u; \xi) \nabla p_{\msfLambda}(u) \, \dd u = \frac{1}{\delta} \int_{\RR^d} \nabla_u \hat{f}(\theta+\delta u;\xi) p_{\msfLambda} (u) \, \dd u \nonumber \\
    &= \frac{1}{\delta} \int_{\RR^d} (\delta I_d)^\intercal \nabla_\theta \hat{f}(\theta+\delta u;\xi) \, p_{\msfLambda} (u) \, \dd u = \nabla_\theta \left( \int_{\RR^d} \hat{f}(\theta+\delta u;\xi) \, \msfLambda (\dd u) \right) = \nabla_\theta \hat{f}_{\delta,\msfLambda}(\theta;\xi), \nonumber
\end{align}
where the third line follows from the definition of $\hat{f}(\theta;\xi)$, and the fourth line is by Stein's identity with Gaussian tail.
Hence,
\begin{align*}
    &\quad \left\|\E_{u \sim \msfLambda, \xi^\pm \sim \msfP_{\theta \pm \delta u(\xi,\cdot)}} [G(\theta;\xi^\pm, u,\delta)] - \E_{\xi' \sim \msfP_\theta(\xi ,\cdot)}[H(\theta;\xi')] \right\| = \left\|\nabla \hat{f}_{\delta,\msfLambda}(\theta;\xi) - \nabla \hat{f}(\theta;\xi) \right\| \\
    &=
    \left\| \frac{1}{(2\pi)^{d/2}} \int_{\RR^d} \left( \frac{\hat{f}(\theta+\delta u;\xi) - \hat{f}(\theta;\xi)}{\delta} - \langle \nabla \hat{f}(\theta;\xi), u \rangle \right) u e^{-\frac{1}{2}\|u\|^2} \, \dd u \right\| \\
    &\le
    \frac{1}{(2\pi)^{d/2} \delta} \int_{\RR^d} \left| \hat{f}(\theta+\delta u;\xi) - \hat{f}(\theta;\xi) - \langle \nabla \hat{f}(\theta;\xi), \delta u \rangle \right| \|u\| e^{-\frac{1}{2}\|u\|^2} \, \dd u  \\
    &\le \frac{1}{(2\pi)^{d/2} \delta} \int_{\RR^d} \frac{L_{\hat{f}}}{2} \|\delta u\|^2 \|u\| e^{-\frac{1}{2}\|u\|^2} \, \dd u = \frac{\delta L_{\hat{f}}}{2} \E[\|u\|^3] \\
    &\le \frac{\delta L_{\hat{f}}}{2} (d+3)^{3/2},
\end{align*}
where the second line uses~\eqref{eq:lem-ec3-fhat}, $\E_{u \sim \msfLambda}[u] = 0$ and $\E_{u \sim \msfLambda}[uu^\intercal] = I_d$. The second inequality follows directly from the $L_{\hat{f}}$-smoothness of $\hat{f}(\theta;\xi)$ by \eqref{eq:lem-ec3-Lf}. 
The last inequality follows from an extension of Wendel's inequality to all $s > 1$, namely $\frac{\Gamma(x+s)}{\Gamma(x)} \le (x+s)^s$ for all $x > 0$.
For integer $s = n$, this is immediate since $\frac{\Gamma(x+n)}{\Gamma(x)} = x(x+1)\cdots(x+n-1) \le (x+n)^n$.
For general $s = n + r$ with $n = \lfloor s \rfloor$ and $0 \le r < 1$, $\frac{\Gamma(x+s)}{\Gamma(x)}
= \frac{\Gamma(x+n+r)}{\Gamma(x+n)} \frac{\Gamma(x+n)}{\Gamma(x)} \le (x+n)^{r} (x+n)^{n}
\le (x+s)^s$ by the log-convexity of Gamma function. Applying this inequality yields $\E[\|u\|^p] = 2^{p/2} \frac{\Gamma ((d+p)/2)}{\Gamma(d/2)} \le 2^{p/2} \left(\frac{d+p}{2}\right)^{p/2} \le (d+p)^{p/2}$ for all $p \ge 3$.
\Halmos \endproof

\subsection{Proof of Theorem~\ref{thm:consistency}}  \label{ec:pf-asymp}

\proof{Proof of Theorem~\ref{thm:consistency}.}
Define the natural filtration associated with the OTL algorithm by $\scrF_k = \sigma(\theta_s, \xi_s: 0 \le s \le k; u_s, \xi_s^\pm: 0 \le s \le k-1)$ for $k \ge 1$ with $\scrF_0 = \sigma (\theta_0, \xi_0)$. The sequence $\{\scrF_k\}_{k \ge 0}$ forms an increasing family of $\sigma$-algebras generated by all random variables observed up to iteration $k$.

Recall the update rule in the OTL algorithm (see Section~\ref{sec:otl}):
\begin{align} \label{eq:tt-zosa-pf}
    \theta_{k+1} = \Pi_\Theta \left(\theta_k - \eta_k G(\theta_k;\xi_k^\pm, u_k, \delta_k) \right).
\end{align}
It can be equivalently written as $\theta_{k+1} = \Pi_\Theta \left( \theta_k - \eta_k \left( \nabla f(\theta_k) + A_k + B_k + V_{k+1} \right) \right)$,
where 
\begin{align}
    A_k &= \E[G(\theta_k;\xi_k^\pm,u_k,\delta_k)|\scrF_{k}] - H(\theta_k;\xi_{k+1}), \label{eq:approx-bias} \\
    B_k &= H(\theta_k;\xi_{k+1}) - \nabla f (\theta_k), \label{eq:markov-bias} \\
    V_{k+1} &= G(\theta_k;\xi_k^\pm,u_k,\delta_k) - \E[G(\theta_k;\xi_k^\pm,u_k,\delta_k)|\scrF_{k}]. \label{eq:mds} 
\end{align}
Here, $A_k$ denotes the approximation bias from the zeroth-order method; $B_k$ denotes the Markovian bias associated with the decision-dependent transition kernel; and $V_{k+1}$ denotes the martingale difference noise.

Define the sequence of time points $t_k$ by $t_0=0$ and $t_k = \sum_{s=0}^{k-1}\eta_s$ for $k\ge 1$. For any (continuous) time $t \ge 0$, let $m(t)$ denote the iterate index identifying the discrete time interval $[t_k, t_{k+1})$; that is, $m(t) = \max\{k \ge 0: t_k \le t\}$. Our goal is to show that the terms $A_k, B_k$ and $V_{k+1}$ are asymptotically negligible, allowing the recursion to be analyzed through the limiting ordinary differential equation (ODE).

\textbf{Approximation bias.} From~\eqref{eq:approx-bias}, we decompose $A_k$ into two terms:
\begin{align*}
    A_k &= \underbrace{\E \left[ \left. \frac{F(\theta_k + \delta_k u_k; \xi_k^+) - F(\theta_k - \delta_k u_k; \xi_k^-)}{2\delta_k} u_k \right\vert \scrF_{k}\right] - \E_{\xi_{k+1} \sim \msfP_{\theta_k}(\xi_{k} ,\cdot)}[H(\theta_k;\xi_{k+1})|\scrF_{k}]}_{\coloneqq A_k^{\Tone}} \\
    &\quad + \underbrace{\E_{\xi_{k+1} \sim \msfP_{\theta_k}(\xi_k ,\cdot)}[H(\theta_k;\xi_{k+1})|\scrF_{k}] - H(\theta_k;\xi_{k+1})}_{\coloneqq A_k^{\Ttwo}}.
\end{align*}

(1) Control term $A_k^{\Tone}$. By Lemma~\ref{lem:approx-bd} and Assumption~\ref{ass:stepsize}, we have 
\begin{align*}
    \left\| \sum_{k=0}^\infty \eta_k A_k^{\Tone} \right\| \le \sum_{k=0}^\infty \eta_k \left\|\nabla \hat{f}_{\delta_k,\msfLambda}(\theta_k;\xi_{k}) - \nabla \hat{f}(\theta_k;\xi_{k}) \right\| &\le  \frac{L_{\hat{f}}}{2} (d+3)^{3/2} \sum_{k=0}^\infty \eta_k  \delta_k < \infty,
\end{align*}
since the bound in Lemma~\ref{lem:approx-bd} holds for all $\xi_k \in \Xi$.
Since the series converges absolutely, we then have $\eta_k A_k^{\Tone} \to 0$ as $k \to \infty$.

(2) Control term $A_k^{\Ttwo}$. Define $D_k = \sum_{s=0}^{k-1} \eta_s A_{s}^{\Ttwo}$ for $ k \ge 1$, and $D_0 = 0$. Since $\E[A_k^{\Ttwo}|\scrF_k] = 0$,  it follows that $\E[D_{k+1}|\scrF_k] = \E[D_{k} + \eta_k A_k^{\Ttwo}|\scrF_k] = D_k$, and thus $(D_k, \scrF_k)$ is a martingale.
Its quadratic variation satisfies
\begin{align}
    &\quad \sum_{k=0}^\infty \E[\|D_{k+1} - D_k\|^2 |\scrF_k] = \sum_{k=0}^\infty \eta_k^2 \E[\| A_{k}^{\Ttwo} \|^2 \vert \scrF_k] \nonumber \\
    &= \sum_{k=0}^\infty \eta_k^2 \E\left[\left. \left\| \E_{\xi_{k+1} \sim \msfP_{\theta_k}(\xi_{k} ,\cdot)}[H(\theta_k;\xi_{k+1})|\scrF_k] - H(\theta_k;\xi_{k+1}) \right\|^2 \right\vert \scrF_k\right] \nonumber \\
    &\le 4\sum_{k=0}^\infty \eta_k^2 \E_{\xi_{k+1} \sim \msfP_{\theta_k}(\xi_{k} ,\cdot)} \left[\left\| H(\theta_k;\xi_{k+1}) \right\|^2 \right] \le 4 C_H \sum_{k=0}^\infty \eta_k^2 < \infty, \label{eq:lem-ec3-A2}
\end{align}
where 
the last inequality follows from Assumption~\ref{ass:peq-sol}-\ref{ass:mom-H}, together with the fact that $\eta_k$ decays to zero faster than $\delta_k$
under Assumption~\ref{ass:stepsize}. By martingale convergence theorem, $D_k$ converges almost surely and in $L^2$. By Doob's $L^2$ maximal inequality, there exists a constant $\wtC_0 > 0$ such that for any integer $j \ge 0$ and any finite $T > 0$, 
\begin{align*}
    \E \left[ \sup_{m(jT) \le k \le m((j+1)T)}  \left\| D_{k} - D_{m(jT)} \right\|^2 \right] &\le \wtC_0 \E \left[ \left\| D_{m((j+1)T)} - D_{m(jT)} \right\|^2 \right] 
    \to 0,
\end{align*}
as $j \to \infty$. By Jensen's inequality,
\begin{align*}
    \E \left[ \sup_{m(jT) \le k \le m((j+1)T)}  \left\| D_{k} - D_{m(jT)} \right\| \right] \le \left( \E \left[ \sup_{m(jT) \le k \le m((j+1)T)}  \left\| D_{k} - D_{m(jT)} \right\|^2 \right] \right)^{1/2} \to 0,
\end{align*}
as $j \to \infty$. Given an $\epsilon > 0$, define the events $E_j(\epsilon) \coloneqq \left\{ \sup_{m(jT) \le k \le m((j+1)T)} \left\| D_{k} - D_{m(jT)} \right\| \ge \epsilon \right\}$. By Markov's inequality, for any $\epsilon > 0$, 
\begin{align*}
    \pr \left( E_j(\epsilon) \right) \le \epsilon^{-1} \E\left[ \sup_{m(jT) \le k \le m((j+1)T)} \|D_{k} - D_{m(jT)} \| \right],
\end{align*}
which is summable by similar arguments as in~\eqref{eq:lem-ec3-A2}. Then by Borel-Cantelli lemma, $\pr \left( E_j(\epsilon) \text{~i.o.} \right) = 0$, so almost surely, $\sup_{m(jT) \le k \le m((j+1)T)} \left\| D_{k} - D_{m(jT)} \right\| \to 0$ as $j \to \infty$. Equivalently, 
\begin{align*}
    \lim_{k \to \infty} \left( \sup_{j \ge k} \max_{t \le T} \left\|\sum_{s=m(jT)}^{m(jT+t)-1} \eta_s A_{s}^{\Ttwo} \right\| \right) = 0, \quad \as
\end{align*}
Hence, the approximation bias is asymptotically negligible almost surely.

\textbf{Markovian bias.} 
From~\eqref{eq:peq} and~\eqref{eq:markov-bias}, we decompose $B_k$ into three terms:
\begin{align*}
    &\quad B_k = H(\theta_k;\xi_{k+1}) - \nabla f (\theta_k) = \nu_{\theta_k}(\xi_{k+1}) - \msfP_{\theta_k} \nu_{\theta_k}(\xi_{k+1}) \\
    &= \underbrace{\left(\nu_{\theta_k}(\xi_{k+1}) - \msfP_{\theta_k} \nu_{\theta_k}(\xi_{k})\right)}_{\coloneqq B_{k}^{\Tone}} + \underbrace{\left(\msfP_{\theta_k} \nu_{\theta_k}(\xi_{k}) - \msfP_{\theta_{k+1}} \nu_{\theta_{k+1}}(\xi_{k+1}) \right)}_{\coloneqq B_{k}^{\Ttwo}} + \underbrace{\left(\msfP_{\theta_{k+1}} \nu_{\theta_{k+1}}(\xi_{k+1}) - \msfP_{\theta_k} \nu_{\theta_k}(\xi_{k+1})\right)}_{\coloneqq B_{k}^{\Tthree}},
\end{align*}
where $B^{\Tone}_k, B^{\Ttwo}_k$, and $B^{\Tthree}_k$ represent martingale, coboundary, and residual term. 
Our goal is to show that there exists a constant $T > 0$ such that
\begin{align}
\lim_{k \to \infty} \left( \sup_{j \ge k} \max_{t \le T} \left\|\sum_{s=m(jT)}^{m(jT+t)-1} \eta_{s+1} B_{s}^{(\text{i})} \right\| \right) = 0, \quad \as \quad \forall (\text{i}) \in \{\Tone,\Ttwo,\Tthree\}. \label{eq:vanish-mix-err}
\end{align}
For convenience, define $S_{k}^{(\text{i})} = \sum_{s=0}^{k-1} \eta_{s+1} B_{s}^{(\text{i})}$ for $(\text{i}) \in \{\Tone,\Ttwo,\Tthree\}$.

(1) Analyze martingale term $B_{k}^{\Tone}$. It suffices to show that for some $T > 0$,
\begin{align*}
    \lim_{j \to \infty} \left( \sup_{m(jT) \le k \le m((j+1)T)} \left\| S_{k}^{\Tone} - S_{m(jT)}^{\Tone} \right\| \right) = 0, \quad \as
\end{align*}
Note that $\E[B_{k}^{\Tone}|\scrF_{k}] = \E[\nu_{\theta_k}(\xi_{k+1})|\scrF_{k}] - \msfP_{\theta_k} \nu_{\theta_k}(\xi_{k}) = 0$, so $S_{k}^{\Tone}$ is a martingale adapted to $\scrF_k$. By
Doob's $L^2$ maximal inequality, there exists a constant $\wtC_1 > 0$ such that
\begin{align*}
    &\quad \E \left[ \sup_{m(jT) \le k \le m((j+1)T)}  \left\| S_{k}^{\Tone} - S_{m(jT)}^{\Tone} \right\|^2 \right] \le \wtC_1 \E \left[ \left\| S_{m((j+1)T)}^{\Tone} - S_{m(jT)}^{\Tone} \right\|^2 \right] \\
    &\le \wtC_1 \E \left[ \sum_{s=m(jT)}^{m((j+1)T)-1} \eta_{s+1}^2 \|B_s^{\Tone}\|^2 \right] = \wtC_1 \sum_{s=m(jT)}^{m((j+1)T)-1} \eta_{s+1}^2 \E [\| \nu_{\theta_s}(\xi_{s+1}) - \msfP_{\theta_s} \nu_{\theta_s}(\xi_{s})\|^2] \\
    &\le 4\wtC_1 \sum_{s=m(jT)}^{m((j+1)T)-1} \eta_{s+1}^2 \E [ \|\nu_{\theta_{s}}(\xi_{s+1})\|^2 ]\\
    &\le 4\wtC_1 L_{\nu} \sum_{s=m(jT)}^{m((j+1)T)-1} \eta_{s+1}^2 \E [ (1+\|\xi_{s+1}\|)^2 ] \coloneqq \wtC_2 \sum_{s=m(jT)}^{m((j+1)T)-1} \eta_{s+1}^2 \E \left[ \|\xi_{s+1}\|^2 \right] \\
    &\le \wtC_2 \wtC_3 \sum_{s=m(jT)}^{m((j+1)T)-1} \eta_{s+1}^2 \to 0,
\end{align*}
as $j \to \infty$ since $\sum_{k=0}^\infty \eta_k^2 < \infty$, where 
the last inequality follows from the drift condition $\E[\|\xi_{k+1}\|^2 \vert \scrF_k] \le C_V\|\xi_k\|^2 + C_V'$. Taking the total expectation and unrolling the recurrence, we have
\begin{align}
    \E [\|\xi_{k+1}\|^2] \le C_V^{k+1} \E[\|\xi_{0}\|^2] + C_V' \sum_{s=0}^{k} C_V^s \le C_V \E[\|\xi_{0}\|^2]  + \frac{C_V'}{1-C_V} \coloneqq \wtC_3, \label{eq:pf-xi2}
\end{align}
which holds for all $k \ge 0$ since $0 < C_V < 1$.
Again, by Jensen's inequality,
\begin{align*}
    \E \left[ \sup_{m(jT) \le k \le m((j+1)T)}  \left\| S_{k}^{\Tone} - S_{m(jT)}^{\Tone} \right\| \right] \le \left(\E \left[ \sup_{m(jT) \le k \le m((j+1)T)}  \left\| S_{k}^{\Tone} - S_{m(jT)}^{\Tone} \right\|^2 \right] \right)^{1/2} \to 0,
\end{align*}
as $j \to \infty$.
The remaining steps proceed by repeating the same reasoning as in the analysis for $\eta_k A_k^{\Ttwo}$, and is therefore omitted for brevity.

(2) Analyze coboundary term $B_k^{\Ttwo}$. We exploit the telescoping structure:
\begin{align*}
    S_{k}^{\Ttwo} &= \sum_{s=0}^{k-1} \eta_{s+1} B_{s}^{\Ttwo} = \sum_{s=0}^{k-1} \eta_{s+1} \left(\msfP_{\theta_s} \nu_{\theta_s}(\xi_{s}) - \msfP_{\theta_{s+1}} \nu_{\theta_{s+1}}(\xi_{s+1}) \right) \\
    &= \eta_1 \msfP_{\theta_0} \nu_{\theta_0} (\xi_0) - \left[ \sum_{s=1}^{k-1} (\eta_s - \eta_{s+1}) \msfP_{\theta_{s}} \nu_{\theta_{s}} (\xi_{s}) \right] - \eta_{k} \msfP_{\theta_{k}} \nu_{\theta_{k}}(\xi_{k}).
\end{align*}
For any $n \ge k$, it follows that
\begin{align*}
    S_{n}^{\Ttwo} - S_{k}^{\Ttwo} &= - \left[ \sum_{s=k}^{n-1} (\eta_s - \eta_{s+1}) \msfP_{\theta_{s}} \nu_{\theta_{s}} (\xi_{s}) \right] - \eta_n \msfP_{\theta_{n}} \nu_{\theta_{n}}(\xi_{n}) + \eta_k \msfP_{\theta_{k}} \nu_{\theta_{k}}(\xi_{k}).
\end{align*}
By Lemma~\ref{lem:trans-peq-sol}-\ref{lem:trans-peq-sol-lingrowth}, we have
\begin{align*}
    \|S_{n}^{\Ttwo} - S_{k}^{\Ttwo}\| \le \underbrace{L_{\msfP\nu} \left[ \sum_{s=k}^{n-1} (\eta_s - \eta_{s+1}) (1+\|\xi_{s}\|) \right]}_{\coloneqq \Tone} + \underbrace{\left\| \eta_n \msfP_{\theta_{n}} \nu_{\theta_{n}}(\xi_{n}) \right\| + \left\| \eta_k \msfP_{\theta_{k}} \nu_{\theta_{k}}(\xi_{k}) \right\|}_{\coloneqq \Ttwo}.
\end{align*}
The term $\Ttwo$ is asymptotically negligible almost surely since
\begin{align*}
    \E\left[\sum_{s=0}^\infty \left\|\eta_s \msfP_{\theta_{s}} \nu_{\theta_{s}}(\xi_{s}) \right\|^2\right] &\le L_{\msfP\nu} \E\left[\sum_{s=0}^\infty \eta_s^2 (1+\|\xi_s\|)^2 \right] 
    \le L_{\msfP\nu} ( 1+\wtC_3)^2 \sum_{s=0}^\infty \eta_s^2  
    < \infty,
\end{align*}
by the bound in~\eqref{eq:pf-xi2}. Then, Tonelli's theorem yields $\sum_{s=0}^\infty \E[\left\|\eta_s \msfP_{\theta_{s}} \nu_{\theta_{s}}(\xi_{s}) \right\|^2] < \infty$ and thus $\left\|\eta_s \msfP_{\theta_{s}} \nu_{\theta_{s}}(\xi_{s}) \right\| \to 0$ almost surely as $s \to \infty$. 

For the term $\Tone$, we have
\begin{align*}
    \E \left[ \sum_{s=1}^\infty  (\eta_s - \eta_{s+1}) (1+\|\xi_{s}\|) \right] \le (1+\wtC_3) \sum_{s=1}^\infty  (\eta_s - \eta_{s+1}) \le \wtC_4 \eta_1 < \infty.
\end{align*}
where the first inequality follows from~\eqref{eq:pf-xi2}. Since $\eta_s > \eta_{s+1}$, Tonelli's theorem applies. 
Therefore, 
\begin{align*}
    \sup_{n \ge k} \left\| \sum_{s=k}^{n-1} \eta_{s+1} B_{s}^{\Ttwo} \right\| = \sup_{n \ge k} \, \|S_{n}^{\Ttwo} - S_{k}^{\Ttwo}\| 
    \lesssim \sup_{n \ge k} \, \sum_{s=k}^{n-1} (\eta_s - \eta_{s+1}) (1+\|\xi_{s}\|) \to 0, 
\end{align*}
almost surely as $k \to \infty$. 

(3) Analyze residual term $B_{k}^{\Tthree}$. 
By Lemma~\ref{lem:trans-peq-sol}-\ref{lem:trans-peq-sol-theta},
\begin{align*}
    \|B^{\Tthree}_{k+1}\| &= \left\| \msfP_{\theta_{k+1}} \nu_{\theta_{k+1}}(\xi_{k+1}) - \msfP_{\theta_k} \nu_{\theta_k}(\xi_{k+1})\right\| \le L_{\msfP\nu} \|\theta_{k+1}-\theta_k\| (1+\|\xi_{k+1}\|) \\
    &= L_{\msfP\nu} \|\Pi_\Theta \left( \theta_k - \eta_k G(\theta_k;\xi_k^\pm,u_k,\delta_k) \right) - \Pi_\Theta (\theta_k) \| (1+\|\xi_{k+1}\|) \\
    &\le L_{\msfP\nu} \left\|\left( \theta_k - \eta_k \frac{F(\theta_k + \delta_k u_k; \xi_{k}^+) - F(\theta_k - \delta_k u_k; \xi_{k}^-)}{2\delta_k} u_k \right) - \theta_k \right\| (1+\|\xi_{k+1}\|) \\ 
    &\le L_{\msfP\nu} C_F \left( \frac{\eta_k}{\delta_k} \right) \|u_k\| (1+\|\xi_{k+1}\|),
\end{align*}
where the second inequality is by the nonexpansiveness of projection $\Pi_\Theta$, and the last inequality is by Assumption~\ref{ass:rand-F}. By Markov's inequality, for any $\epsilon > 0$, 
\begin{align*}
    \pr ( \|B^{\Tthree}_{k+1}\| > \epsilon ) &\le \left(\frac{L_{\msfP\nu} C_F}{\epsilon} \right)^2 \left( \frac{\eta_k}{\delta_k} \right)^2 \E[\|u_k\|^2 (1+\|\xi_{k+1}\|)^2] \\
    &\le \left(\frac{L_{\msfP\nu} C_F}{\epsilon} \right)^2  2d(1+\wtC_3) \left( \frac{\eta_k}{\delta_k} \right)^2 \to 0,
\end{align*}
where the last inequality is due to~\eqref{eq:pf-xi2}, and $\eta_k/\delta_k \to 0$ as $k \to \infty$. Here, $\pr ( \|B^{\Tthree}_{k+1}\| > \epsilon )$ is summable since $\sum_{k=0}^\infty \eta_k^2/\delta_k^2 < \infty$. Then by the Borel-Cantelli lemma, we have $\|B^{\Tthree}_{k+1}\| \to 0$ as $k \to \infty$ almost surely. Since $\eta_k \to 0$ as $k \to \infty$, the weighted tail sum is also asymptotically negligible; that is, as $k \to \infty$,
\begin{align*}
    \sup_{n \ge k} \left\| \sum_{s=k}^{n-1} \eta_{s} B_{s}^{\Tthree} \right\| \to 0 \quad \as
\end{align*}
Hence, the Markovian bias is asymptotically negligible almost surely.

\textbf{Martingale difference sequence.}
From~\eqref{eq:mds}, $V_{k+1}$ is $\scrF_{k+1}$-measurable and satisfies $\E[V_{k+1}|\scrF_k] = 0$ almost surely. Define $M_k = \sum_{s=0}^{k-1} \eta_s V_{s+1}$ for $ k \ge 1$, and $M_0 = 0$. It follows that $M_{k+1} -M_k$ is a martingale difference sequence. 
Its quadratic variation is given by
\begin{align*}
    &\quad \sum_{k=0}^\infty \E[\|M_{k+1} - M_k\|^2 |\scrF_k] = \sum_{k=0}^\infty \eta_k^2 \E[\|V_{k+1}\|^2 |\scrF_k] \\
    &= \sum_{k=0}^\infty \eta_k^2 \E[\|G(\theta_k;\xi_k^\pm,u_k,\delta_k) - \E[G(\theta_k;\xi_k^\pm,u_k,\delta_k)|\scrF_k] \|^2 \vert \scrF_k] \le 4 \sum_{k=0}^\infty \eta_k^2 \E[\|G(\theta_k;\xi_k^\pm,u_k,\delta_k)\|^2|\scrF_k] \\
    &= \sum_{k=0}^\infty \frac{\eta_k^2 }{\delta_k^2} \E\left[ \left.\left(F(\theta_k + \delta_k u_k; \xi_k^+) - F(\theta_k - \delta_k u_k; \xi_k^-)\right)^2 \|u_k\|^2 \right\vert \scrF_k\right] \\
    &\le \sum_{k=0}^\infty \frac{\eta_k^2 }{\delta_k^2} \E\left[ \left. (2C_F)^2 \|u_k\|^2 \right\vert \scrF_k\right] \le 4d C_F^2 \sum_{k=0}^\infty \frac{\eta_k^2 }{\delta_k^2} < \infty,
\end{align*}
where the last inequality is by Assumption~\ref{ass:rand-F}.
Consequently, $M_k$ converges almost surely and in $L^2$. Then by Doob's $L^2$ maximal inequality,
$\sup_{m(jT) \le k \le m((j+1)T)} \| M_{k} - M_{m(jT)} \| \to 0$ as $j \to \infty$
almost surely for any finite $T>0$. Using the arguments similar to those for $\eta_k A_k^{\Ttwo}$, we have 
\begin{align*}
    \lim_{k \to \infty} \left( \sup_{j \ge k} \max_{t \le T} \left\|\sum_{s=m(jT)}^{m(jT+t)-1} \eta_s V_{s+1} \right\| \right) = 0, \quad  \as
\end{align*}
Hence, the martingale difference sequence is asymptotically negligible almost surely.

\textbf{Limiting ODE.}
The continuous-time interpolation of the sequence $\{\theta_k\}_{k \ge 0}$ is defined by $\theta(t_k)=\theta_k$ for $t\in[t_k,t_{k+1})$ with $k\ge 0$.
This construction yields a continuous-time interpolated trajectory that approximates the discrete-time dynamics of the OTL algorithm. In the limit as the stepsize $\eta_k \to 0$, the evolution of this interpolated trajectory is described by the ODE:
\begin{align}
    \dot{\theta}(t) = \overline{\Pi}_\Theta\left(-\nabla f(\theta(t))\right), \label{eq:ode}
\end{align}
where $\overline{\Pi}_\Theta(\cdot)$ denotes the directional derivative 
of the projection operator $\Pi_\Theta(\cdot)$. 
For any continuous function $\psi: \Theta \mapsto \RR^d$
and any point $\vartheta\in\Theta$, 
it is given by $\overline{\Pi}_\Theta(\psi(\vartheta)) = \lim_{\epsilon \downarrow 0} \left( (\Pi(\vartheta + \epsilon \psi(\vartheta)) - \vartheta) / \epsilon \right)$. Since $\Theta$ is compact and convex, the limit always exists. In particular, if $\vartheta$ lies in the interior of $\Theta$, it follows trivially that $\overline{\Pi}(\psi(\vartheta)) = \psi(\vartheta)$. If $\vartheta$ is located on the smooth boundary of $\Theta$, the projection behaves locally as the orthogonal projection of $\psi(\vartheta)$ onto the tangent cone $\scrT_\Theta(\vartheta)$; i.e., $\overline{\Pi}(\psi(\vartheta)) = \Pi_{\scrT_\Theta(\vartheta)}(\psi(\vartheta))$. In both cases, the limit is unique. $\overline{\Pi}(\cdot)$ may be set-valued when $\vartheta$ is not regular; i.e., a nonsmooth point of the boundary.

For any trajectory of the projected ODE in~\eqref{eq:ode} and any regular point $\theta(t)$, we have
\begin{align*}
    \frac{d}{dt} f(\theta(t)) = \nabla f(\theta)(t)^\intercal \dot\theta (t) = -\|\Pi_{\scrT_\Theta(\theta)}(\nabla f(\theta(t)))\|^2 \le 0.
\end{align*}
Hence, $f(\theta(t))$ is non-increasing and thus a Lyapunov function for this projected gradient flow. Under Assumption~\ref{ass:feasible-set},  all trajectories are automatically bounded and remain within $\Theta$.
By LaSalle's invariance principle, every bounded trajectory converges to the largest invariant subset of $\Theta_0 \coloneqq \left\{\theta\in\Theta: \Pi_{\scrT_\Theta(\theta)}(\nabla f(\theta))=0\right\} = \left\{\theta\in\Theta: 0 \in \nabla f(\theta) + \scrN_{\Theta}(\theta)\right\}$, which is the set of all locally asymptotically stable points. For any given $\theta^* \in \Theta_0$, its attraction domain $D(\theta^*) = \{ \theta_0 \in \Theta: \lim_{t \to \infty} \theta(t;\theta_0) = \theta^*\}$ is open. Therefore, there exists a small ball $\BB(\theta^*,r) \subset D(\theta^*)$ containing infinitely many iterates $\theta_k$. Under Assumption~\ref{ass:stepsize} where $\sum_{k=0}^\infty \eta_k = \infty$ and $\sum_{k=0}^\infty \eta_k^2 < \infty$, we conclude that $\lim_{k \to \infty} \theta_k = \theta^*$. This establishes that when the trajectory of converges to $\theta^*$, and the stochastic recursion \eqref{eq:tt-zosa-pf} asymptotically tracks this ODE, then the sequence of the OTL algorithm iterates $\{\theta_k\}_{k \ge 0}$ also converges to $\theta^*$.
\Halmos \endproof

\end{document}